%% file: main.tex
\documentclass[10pt]{article} 
\usepackage[accepted]{rlc}

\usepackage{amssymb, amsthm, tikz-cd}            
\usepackage{mathtools}          
\usepackage{mathrsfs}           
\usepackage{graphicx}           
\usepackage[space]{grffile}     
\usepackage{url}                
\usepackage{subfigure}
\usepackage{amsmath}

\usepackage{dsfont}
\usepackage{algorithmicx, algpseudocode, algorithm, enumitem}

\newtheorem{theorem}{Theorem}[section]
\newtheorem{proposition}[theorem]{Proposition}
\newtheorem{lemma}[theorem]{Lemma}
\newtheorem{corollary}[theorem]{Corollary}
\newtheorem{example}[theorem]{Example}
\newtheorem{definition}[theorem]{Definition}
\newtheorem{assumption}[theorem]{Assumption}
\newtheorem{remark}[theorem]{Remark}

\newcommand{\Ac}{\mathcal{A}}
\newcommand{\St}{\mathcal{S}}

\newcommand{\RSt}{\mathbb{R}^{\mathcal{S}}}
\newcommand{\RAS}{\mathbb{R}^{\mathcal{A} \times \mathcal{S}}}
\newcommand{\RSA}{\mathbb{R}^{\mathcal{S} \times \mathcal{A}}}
\newcommand{\BestResponse}{{\fontfamily{qcr} \selectfont{Best-response}}}
\newcommand{\OnlineAlgo}{{\fontfamily{qcr} \selectfont{Online-Algorithm}}}
\DeclareMathOperator*{\argmax}{argmax}

\definecolor{lxs}{RGB}{138,43,226}

\allowdisplaybreaks

\input{macro}

\title{Distributionally Robust Constrained  \\
Reinforcement Learning under Strong Duality}

\author{Zhengfei Zhang\thanks{This work was initiated during the visiting undergraduate research program at the California Institute of Technology. } \\
    zhengfei.zhang@sc.tsinghua.edu.cn\\
    Tsinghua University
    \And 
    Kishan Panaganti\\
    kpb@caltech.edu\\
    California Institute of Technology
    \And 
    Laixi Shi\\
    laixis@caltech.edu\\
    California Institute of Technology
    \And 
    Yanan Sui\\
    ysui@tsinghua.edu.cn\\
    Tsinghua University
    \And 
    Adam Wierman\\
    adamw@caltech.edu\\
    California Institute of Technology
    \And 
    Yisong Yue\\
    yyue@caltech.edu\\
    California Institute of Technology
} 

\begin{document}

\maketitle

\begin{abstract}

We study the problem of Distributionally Robust Constrained RL (DRC-RL), where the goal is to maximize the expected reward subject to environmental distribution shifts and constraints.  This setting captures situations where training and testing environments differ, and policies must satisfy constraints motivated by safety or limited budgets. Despite significant progress toward algorithm design for the separate problems of distributionally robust RL and constrained RL, there do not yet exist algorithms with end-to-end convergence guarantees for DRC-RL. 
We develop an algorithmic framework based on strong duality that enables the first efficient and provable solution in a class of environmental uncertainties. 
Further, our framework exposes an inherent structure of DRC-RL that arises from the combination of distributional robustness and constraints, which prevents a popular class of iterative methods from tractably solving DRC-RL, despite such frameworks being applicable for each of distributionally robust RL and constrained RL individually.  
Finally, we conduct experiments on a car racing benchmark to evaluate the effectiveness of the proposed algorithm.




\end{abstract}

\input{01-introduction}
\input{02-formulation}
\input{03-framework}

\input{04-contamination}

\input{05-general}

\input{06-conclusion}


\bibliography{main}
\bibliographystyle{rlc}


\input{07-appendix}

\end{document}

%% file: macro.tex
\newcommand{\cS}{{\mathcal{S}}}
\newcommand{\cT}{{\mathcal{T}}}

\usepackage{scalerel,stackengine}
\stackMath
\newcommand\reallywidehat[1]{%
\savestack{\tmpbox}{\stretchto{%
  \scaleto{%
    \scalerel*[\widthof{\ensuremath{#1}}]{\kern-.6pt\bigwedge\kern-.6pt}%
    {\rule[-\textheight/2]{1ex}{\textheight}}
  }{\textheight}%
}{0.5ex}}%
\stackon[1pt]{#1}{\tmpbox}%
}
\newcommand\reallywidecheck[1]{%
\savestack{\tmpbox}{\stretchto{%
  \scaleto{
    \scalerel*[\widthof{\ensuremath{#1}}]{\kern-.6pt\bigwedge\kern-.6pt}%
    {\rule[-\textheight/2]{1ex}{\textheight}}
  }{\textheight}%
}{0.5ex}}%
\stackon[1pt]{#1}{\scalebox{-1}{\tmpbox}}%
}

%% file: 01-introduction.tex
\section{Introduction}
\label{sec:introduction}

In many real-world decision-making tasks, policies must not only be reward-maximizing, but also be robust to environmental distribution shifts while satisfying application constraints.  Environmental distribution shifts occur in scenarios where there is a mismatch between the training and testing environments, such as due to environment changes \citep{maraun2016bias}, modeling errors \citep{chen1996design}, or adversarial disturbances \citep{pioch2009adversarial}.  Constraints are imposed in tasks that require adherence to safety factors \citep{haddadin2012truly,weidemann2023literature},  budgets in strategy games \citep{vinyals2019grandmaster}, diverse interests in advertisement recommendations \citep{kinnon2021examining,bagenal2023embracing}, and so on.
This motivates us to tackle both challenges simultaneously, inspiring the study of problems called distributionally robust constrained RL (DRC-RL) \citep{russel2020robust, wang2022robust}. 

The goal of DRC-RL is to learn a policy that simultaneously optimizes the expected reward and satisfies the constraints in the worst-case scenario when the deployed environment deviates from the nominal one within a prescribed uncertainty set. DRC-RL has received growing attention in recent years and is typically modeled as a constrained optimization problem.
As it is unknown if the strong duality holds for the DRC-RL problems, most recent works either use different formulations (e.g., risk-averse) to consider environmental uncertainty \citep{queeney2024risk, kim2024trust}, or simply focus on one of the primal \citep{sun2024constrained} and the dual problem \citep{wang2022robust, bossens2023robust}, such that an end-to-end guarantee is still absent.
More broadly, there has been significant progress in developing rigorous algorithms that address the two challenges that make up DRC-RL individually: distributionally robust RL (DR-RL) \citep{iyengar2005robust, wiesemann2013robust, li2022first, panaganti2022robust} and constrained RL (C-RL) \citep{le2019batch, miryoosefi2019reinforcement,  efroni2020exploration, ding2021provably}.  
Many of these works have focused on a simple, intuitive, greedy policy induced by taking the greedy (best) action with respect to the current learned value functions. This raises the question of whether a similar greedy approach can be effective in DRC-RL or if additional challenges arise from the combination of distributional robustness and constraints.

To address this question, in this paper, we develop a general framework that 
transfers the policy learning problem to a game-theoretic formulation with a constructed strong duality, where the dual problem is treated as a player's objective.
In DR-RL and C-RL, targets similar to this dual function are solved via greedy policies \citep{iyengar2005robust, le2019batch}.
Mathematically, one can think of such a procedure as applying an operator efficiently and greedily, and convergence depends on proving that this operator is a contraction.
While such a greedy approach works for DR-RL \citep{iyengar2005robust} and C-RL \citep{miryoosefi2019reinforcement} in isolation, we show that only with further assumptions can one apply this approach to DRC-RL, e.g., for R-contamination uncertainty sets \citep{huber1965robust, wang2022policy}. We prove that, in general, no such operator exists for the joint DRC-RL problem, implying an impossibility result for a commonly applied class of algorithms. 

In summary, this paper makes the following main contributions:
\vspace{-0.1in}
\begin{itemize}
    \item We propose a multi-level systematic framework to solve DRC-RL for general uncertainty sets in Section \ref{sec:framework}. We show that guarantees for subroutines combine to ensure end-to-end guarantees for DRC-RL.
    \vspace{-0.05in}
    \item Focusing on the R-contamination uncertainty set, we instantiate our framework to provide the first provable efficient solution for DRC-RL in Section \ref{sec:contamination}. 
    Our solution uses a shortened horizon in subroutines to ensure distributional robustness. We verify its effectiveness with an experiment using a high-dimensional car-racing task.
    \vspace{-0.05in}
    \item We consider general uncertainty sets in Section \ref{sec:inherent} and show that the combination of constraints and distributional robustness requirements yields that DRC-RL cannot be solved by considering greedy policies, which is the key of a popular class of iteration methods proposed previously for standard RL, DR-RL, and C-RL problems \citep{iyengar2005robust, le2019batch}. 
\end{itemize}

\textbf{Notation.}
For any set $\cS$, $\Delta(\cS)$ denotes the set of probability distribution over $\cS$.
We use $\otimes_{i}\mathcal{X}_i$ to denote a product space of spaces $\mathcal{X}_i$'s.
We use $\langle x, y \rangle$ to denote the inner product of two same-sized vectors.
We claim $f(T)$ is $o(T)$ if $f(T)/T\to 0$ as $T\to\infty$.  $f(x)\leq O(g(x))$ denote $f(x)\leq C g(x)$ for some positive constant $C$. 

%% file: 02-formulation.tex
\section{Preliminaries and Problem Formulation}\label{sec:problem-formulation}

\paragraph{Robust Markov Decision Process.} A robust Markov decision process (MDP) with infinite horizon can be specified by a tuple $(\St, \Ac, \mathcal{P}, r, \gamma, \mu)$, where $\St$ denotes the finite state space\footnote{We study the finite state space here for simplicity sake, while all theoretical and empirical results in the main text should hold for any bounded and closed finite-dimensional state space as well.}, $\Ac$ denotes the finite action space, $r : \St \times \Ac \to [0, \bar{R}]$ is the known deterministic reward function with some positive maximal magnitude $\bar{R}$, $\gamma \in [0, 1)$ is the discount factor, and $\mu$ is the initial state distribution. In this paper, we focus on $(s, a)$-rectangularity uncertainty set for the transition kernel $\mathcal{P}$ \citep{nilim2003robustness, iyengar2005robust}, i.e. $\mathcal{P} = \otimes_{(s, a) \in \St \times \Ac} \mathcal{P}_{s, a}$, where $\mathcal{P}_{s, a} = \{P_{s, a} \in \Delta(\St) | D(P_{s, a}, P^o_{s, a}) \le \beta_{s, a}\}$, describing a neighborhood of the nominal model $P^o = (P^o_{s, a}, (s, a) \in \St \times \Ac)$ by some measurement function $D : \Delta(\St) \times \Delta(\St) \to \mathbb{R}$ and robustness level $\beta \in \RSA$. 
We consider any policy $\pi : \St \to \Ac$ in the class of deterministic policies $\Pi$. The robust value function of a policy $\pi$ is then defined as the worst-case accumulated reward following the policy $\pi$ over MDPs within the uncertainty set as below \citep[Sec.3]{iyengar2005robust}: 
\begin{equation} \label{robust_value_function}
    V^{\pi}_{r}(s) \coloneqq \min_{\mathcal{K} \in \otimes_{t \ge 0} \mathcal{P}} \mathbb{E}_{\mathcal{K}} [\sum_{t = 0}^{\infty} \gamma^t r(s_t, a_t) | s_0 = s, \pi].
\end{equation} 
Here $\mathcal{K}$ is a sequence of transition kernels within the same uncertainty set $\mathcal{P}$ over each time step.
Moreover, for any policy $\pi$, the robust value function $V^{\pi}_{r}$ is the unique stationary point of the robust Bellman consistency operator \citep{iyengar2005robust}, \begin{equation}
    \mathcal{T}^{\pi}_{r, \mathrm{rob}} v(s) \coloneqq r(s, \pi(s)) + \gamma \min_{P \in \mathcal{P}} \langle P_{s, \pi(s)}, v \rangle.  \label{dr_operator}
\end{equation}
The ultimate goal of distributionally robust RL is to find the optimal robust policy $\pi^*$ that attains maximized robust value function $V^*_{r}=\max_{\pi} V^{\pi}_{r}$. 
To attain the optimal robust value $V^*_{r}$, we have various dynamic programming procedures \citep{iyengar2005robust}, e.g. iterations $v_{k+1} = \mathcal{T}^{*}_{r, \mathrm{rob}} v_{k}$ converge to $V^*_{r}$, where the optimality operator
$\mathcal{T}^{*}_{r, \mathrm{rob}} v(s) \coloneqq \max_{\pi} \mathcal{T}^{\pi}_{r, \mathrm{rob}} v(s) = \mathcal{T}^{\mathcal{G}(v)}_{r, \mathrm{rob}} v(s)$ using the greedy policy $\mathcal{G}(v)[s] \coloneqq \arg\max_a \{r(s, a) + \gamma \min_{P \in \mathcal{P}} \langle P_{s, a}, v \rangle\}$. Typically, such problems that are usually solved by greedy policy $\mathcal{G}(v)[s] = \arg\max_{\pi}\mathcal{T}^{\pi}_{r, \mathrm{rob}} v(s)$ has the name of policy improvement step.

\paragraph{Distributionally Robust Constrained RL (DRC-RL).} 
We formulate the distributionally robust constrained MDP as a tuple $(\St, \Ac, \mathcal{P}, r, g, \tau, \gamma, \mu)$, where $\St, \Ac, \mathcal{P}, r, \gamma, \mu$ are identical to that in robust MDPs. Here, $g := [g_1, g_2,\cdots g_m]$ with $g_i: \St \times \Ac \to [0,\tau_i]$ for all $i\in 1,2,\cdots, m$, representing the aggregation vector of $m$ known deterministic reward-based constraint functions based on the constraint thresholds $\tau = [\tau_1,\cdots, \tau_m]$. We aim to learn a policy $\pi$ within the deterministic policy class denoted as $\Pi$.
As environmental distribution shifts can applied to constraints and the objective independently, e.g. estimation errors, we formulate the goal of distributionally robust constrained RL (DRC-RL) as solving the following constrained optimization problem: \begin{equation}
    \max_{\pi \in \Pi} V^{\pi}_{r}(\mu) ~~\mathrm{s.t.}~~ V^{\pi}_{g_i}(\mu) \ge \tau_i, 1 \le i \le m ,\label{conservative_form}
\end{equation} where $V^{\pi}_r, V^{\pi}_{g_i}$ is the robust value functions (\ref{robust_value_function}) corresponding to the objective reward function $r$ and constraint functions $g_i$'s, and their corresponding expected robust values according to the initial state distributions are $V^{\pi}_r(\mu) = \langle V^{\pi}_r, \mu \rangle$ and $V^{\pi}_{g_i}(\mu) = \langle V^{\pi}_{g_i}, \mu \rangle$. For brevity sake, we denote the constraint vector as $V^{\pi}_{g} := [V^{\pi}_{g_1}, V^{\pi}_{g_2},\cdots, V^{\pi}_{g_m}]^\top \in \mathbb{R}^{m}$.



%% file: 03-framework.tex
\section{DRC-RL with General Uncertainty Sets}
\label{sec:framework}

In this section, we develop a general framework 
and meta algorithm for DRC-RL with an arbitrary uncertainty set in Section \ref{sec:meta} and introduce the subroutines of the framework in Section \ref{sec:best}. 

\subsection{A Meta Algorithm for DRC-RL}
\label{sec:meta}

Constrained RL can be viewed as a constrained optimization problem that has been proven to have strong duality generally by \citet{paternain2019constrained}.  However,
 it is currently not known whether DRC-RL maintains strong duality. To show strong duality of DRC-RL problem \eqref{conservative_form}, we consider a class of mixed policies denoted as $Conv(\Pi)$, defined as below \citep{miryoosefi2019reinforcement, le2019batch}:
 \begin{align}
      \big\{ \pi_{\alpha,\{\pi_i\}_{i=1}^T} \sim \textsf{Categorical} \left( \{\pi_i\}_{i=1}^T, \alpha\right): 0<T < \infty, \pi_i \in \Pi, \forall i; \alpha = [\alpha_1,\cdots, \alpha_T] \in \Delta(T)\big\},
 \end{align} where $\textsf{Categorical}(\{\pi_i\}_{i=1}^T, \alpha)$ is a categorical distribution such that $\mathbb{P}(\pi_{\alpha,\{\pi_i\}_{i=1}^T} = \pi_i) = \alpha_i$ for all $i = 1,2,\cdots, T$.
To execute any mixed policy $\pi_{\alpha,\{\pi_i\}_{i=1}^T} \in  Conv(\Pi)$, at the beginning of each episode,  a deterministic policy $\pi$ is sampled independently from $\textsf{Categorical}(\{\pi_i\}_{i=1}^T, \alpha)$ and serve as the action selection rule for the entire episode. 
Thus, the robust value function of a mixed policy $V^{\pi_{\alpha,\{\pi_i\}_{i=1}^T} }_r$ is defined as $V^{\pi_{\alpha,\{\pi_i\}_{i=1}^T} }_r \coloneqq \mathbb{E}_{\pi \sim \textsf{Categorical}(\{\pi_i\}_{i=1}^T, \alpha) }[V^{\pi}_r] = \sum_{i=1}^T \alpha_i V^{\pi_i}_r$. 


\begin{proposition} \label{policy_convexify}
    When substituting $\Pi$ with its convex hull $Conv(\Pi)$ in the DRC-RL problem (\ref{conservative_form}), strong duality holds if Slater's condition holds. 
\end{proposition}
The proof of the above proposition and other results of this section are postponed to  Appendix \ref{sec:appendix_proof_sec3}.

\begin{algorithm}[t]
\caption{Meta Algorithm for the distributionally robust constrained RL problem}
\label{Meta_Algo}
\begin{algorithmic}[1]
\For{each round \( t \)}
    \State \( \pi_t \gets \text{BestResponse}(\lambda_t) \) \Comment{Non-trivial for DRC-RL problems}
    \State \( \hat{\pi}_t \gets \frac{1}{t} \sum_{t'=1}^{t} \pi_{t'}, \hat{\lambda}_t \gets \frac{1}{t} \sum_{t'=1}^{t} \lambda_{t'} \) \Comment{Mixed policy $\hat{\pi}_t$}
    \State \( L_{\max} = L(\text{BestResponse}(\hat{\lambda}_t), \hat{\lambda}_t) \)
    \State \( L_{\min} = \min_{\lambda} L(\hat{\pi}_t, \lambda) \)
    \If{\( L_{\max} - L_{\min} < \omega \)}
        \State \Return \( \hat{\pi}_t \)
    \EndIf
    \State \( \lambda_{t+1} \gets \text{OnlineAlgo} (\pi_1, ..., \pi_{t-1}, \pi_t) \)
\EndFor
\end{algorithmic}
\end{algorithm}

We assume the DRC-RL problem (\ref{conservative_form}) is feasible and that Slater's condition \citep{boyd2004convex} holds, where the latter only requires the existence of an interior solution upon feasibility. The problem considering the augmented solution class $Conv(\Pi)$ has a solution no worse than the original problem (\ref{conservative_form}), and the convexification itself does not pose any restriction on deterministic policies. As such, we directly denote the convex hull $Conv(\Pi)$ as $\Pi$ in the rest of the paper, and always treat $\pi$ as a mixed policy unless specified.

The Lagrangian of \eqref{conservative_form} is $L(\pi, \lambda) \coloneqq V_{r}^{\pi}(\mu) - \lambda^{\top}(V_g^{\pi}(\mu) - \tau)$ for some $\lambda \in \mathbb{R}_{+}^{m}$. 
Strong duality indicates $\max_{\pi \in \Pi}\min_{\lambda \in \mathbb{R}_{+}^{m}} L(\pi, \lambda) = \min_{\lambda \in \mathbb{R}_{+}^{m}} \max_{\pi \in \Pi} L(\pi, \lambda)$.
By the definition of mixed polices, $V_{r}^{\pi}(\mu)$ and $V_{g}^{\pi}(\mu)$ are all linear to policy $\pi$ (see Appendix \ref{sec:proof_convexify}). Therefore, $L(\pi, \lambda)$ is linear to both $\lambda$ and $\pi$, and a game-theoretic perspective can be applied.  That is, we view the problem as a two-player game between a $\pi$-player and a $\lambda$-player \citep{freund1999adaptive, miryoosefi2019reinforcement}. 

Algorithm \ref{Meta_Algo} describes this repeated game, where both players seek to decrease the duality gap.
The $\pi$-player runs\BestResponse~to maximize Lagrangian $L(\pi, \lambda_t)$ given the current $\lambda_t$, \begin{equation}
    \pi_t \coloneqq \mathrm{Best\mbox{-}response}(\lambda_t) \in \mathrm{argmax}_{\pi} L(\pi, \lambda_t) .\label{best_response}
\end{equation} The $\lambda$-player then employs any no-regret\OnlineAlgo~ \citep{shalev2007online} to minimize $L(\pi_t, \lambda)$, which satisfies: \begin{equation}
    \sum_{t} (-L)(\pi_t, \lambda_t) \ge \max_{\lambda} \sum_{t} (-L)(\pi_t, \lambda) - o(T). \label{no_regret}
\end{equation} 
Algorithm \ref{Meta_Algo} terminates when the estimated primal-dual gap is below a threshold $w$.

\begin{proposition} \label{meta_prop}
Algorithm \ref{Meta_Algo} is guaranteed to converge if (i)\BestResponse~gives the best deterministic policy in the deterministic policy class (ii) $L_{\max}$ and $L_{\min}$   in Algorithm \ref{Meta_Algo} are precisely evaluated. Additionally, the exact convergence rate depends on the regret of\OnlineAlgo.
\end{proposition}

\subsection{The\OnlineAlgo~and\BestResponse~Subroutines}
\label{sec:best}

Given Proposition \ref{meta_prop}, the remaining task is to instantiate the\OnlineAlgo~and\BestResponse~subroutines.  
The requirements for \OnlineAlgo~are standard. Any no-regret (\ref{no_regret}) online optimization algorithm is valid. Examples include Online Gradient Descent \citep{zinkevich2003online}, Exponentiated Gradient \citep{kivinen1997exponentiated}, and Follow-the-Regularized-Leader \citep{shalev2007online}. The\BestResponse~subroutine, which corresponds to the dual function of DRC-RL problem (\ref{conservative_form}), is more difficult to instantiate and currently has no provable method for any specific uncertainty set among related works \citep{mankowitz2020robust, wang2022robust, bossens2023robust}. Therefore, the key challenge is to efficiently and provably solve\BestResponse~problem (\ref{best_response}).

In detail, the\BestResponse~problem with a given $\lambda_t$ corresponds to the maximization problem of value functions of a form that often occurs in RL, e.g. $\max_{\pi}  V^{\pi}_{r}$ in DR-RL,  where \begin{align}
    \pi_t \in&~ \argmax_{\pi \in \Pi} L(\pi, \lambda_t) 
    = \argmax_{\pi \in \Pi} V^{\pi}_{r}(\mu) - \lambda_t^{\top}V^{\pi}_{g}(\mu) .\label{conservative_inner}
\end{align} 
With a finite action space, similar maximization problems can be efficiently solved using iterative methods over greedy policies using some operators in various popular RL problems, such as standard RL \citep{scherrer2015approximate}, distributionally robust RL \citep{iyengar2005robust, derman2021twice, panaganti2022robust}, constrained RL \citep{le2019batch, miryoosefi2019reinforcement}, and regularized RL \citep{geist2019theory}. Using a similar approach, for any policy $\pi$, we propose a consistency operator $\mathcal{T}^{\pi} : v \in \mathbb{R}^{\St} \mapsto \mathcal{T}^{\pi} v \in \mathbb{R}^{\St}$ so that for any given $\lambda_t$, 
\begin{align}
    [\mathcal{T}^{\pi} v](s) =&~ (r - \lambda_t^{\top}g)(s, \pi(s)) + \gamma \langle P^o_{s, \pi(s)}, v \rangle +  \gamma \min_{P \in \mathcal{P}} \langle P_{s, \pi(s)} - P^o_{s, \pi(s)}, V^{\pi}_{r} \rangle  \label{general_operator}  \\
    &~- \gamma \lambda_t^{\top} \min_{P \in \mathcal{P}} \langle P_{s, \pi(s)} - P^o_{s, \pi(s)}, V^{\pi}_{g} \rangle. \nonumber
\end{align} 
where $(r - \lambda_t^{\top}g)(s, \pi(s)) \coloneqq r(s, \pi(s)) - \lambda_t^{\top}g(s, \pi(s))$ for the brevity sake.\footnote{Here $V^{\pi}_r$ and $V^{\pi}_g$ are the robust value functions that are fixed given $\pi$, making $\mathcal{T}^{\pi}$ in (\ref{general_operator}) not a practical operator yet until further specification. }
Correspondingly, an optimality operator $\mathcal{T}^{*}$ with a fixed state $s$ can be defined as $[\mathcal{T}^{*} v](s) = \max_{\pi \in \Pi} [\mathcal{T}^{\pi} v](s)$.
\begin{proposition}
    The consistency and optimality operators, i.e., $\mathcal{T}^{\pi}$ and $\mathcal{T}^{*}$, satisfy: 
        \\(1) Monotonicity: let $v_1, v_2 \in \mathbb{R}^{\mathcal{S}}$ such that $v_1 \ge v_2$, then $\mathcal{T}^{\pi} v_1 \ge \mathcal{T}^{\pi} v_2$ and $\mathcal{T}^{*} v_1 \ge \mathcal{T}^{*} v_2$. 
        \\(2) Transition Invariance: for any $c \in \mathbb{R}$, we have $\mathcal{T}^{\pi} (v + c\textbf{1}) = \mathcal{T}^{\pi} v + \gamma c\textbf{1}$ and $\mathcal{T}^{*} (v + c\textbf{1}) = \mathcal{T}^{*} v + \gamma c\textbf{1}$.
        \\(3) Contraction: The operator $\mathcal{T}^{\pi}$ and $\mathcal{T}^{*}$ are $\gamma$-contractions. Further, $V^{\pi}_{r} - \lambda_t^{\top}V^{\pi}_{g}$ is the unique stationary points of operator $\mathcal{T}^{\pi}$. 
    \label{basic_properties_T}
\end{proposition}
The properties summarized in Proposition \ref{basic_properties_T} allow us to apply the consistency operator in an approximate modified policy iteration (AMPI) scheme \citep{scherrer2015approximate} to solve\BestResponse. AMPI scheme generalizes both value iteration (as used in Section \ref{sec:problem-formulation}) and policy iteration methods and is widely used for other RL problems \citep{scherrer2015approximate, geist2019theory, panaganti2022robust}. The procedure of AMPI can be described  as follows:\begin{equation}
        \pi^{k + 1} = \arg\max_{\pi}\,^{\epsilon'_{k + 1}} \mathcal{T}^{\pi} v^k \quad \text{and} \quad
        v^{k + 1} = (\mathcal{T}^{\pi^{k + 1}})^m v^k + \epsilon_{k + 1}, \label{ampi}
\end{equation}
where $\epsilon_k \in \RSt, \epsilon'_k \in \RSt$ are some optimization errors in the $k$-th iteration. Here, we assume the operator $\max^{\epsilon_{k+1}'}_{\pi} \mathcal{T}^{\pi} v^k$ guarantees $\max_{\pi} [\mathcal{T}^{\pi} v^k](s) \le [\mathcal{T}^{\pi^{k + 1}} v^k](s) + \epsilon'_{k + 1}(s)$ for all $s \in \St$ for now.

In words, the two update rules in \eqref{ampi} correspond to approximate policy improvement and approximate policy evaluation, respectively. 
In RL literature, those two steps can be solved by some oracles. 
Especially, the approximate policy improvement step is often represented as $\pi^{k + 1} = \mathcal{G}^{\epsilon_{k + 1}}(v^k)$, being the greedy policy with respect to $v^k$ and an error term $\epsilon_{k + 1}$ \citep{munos2008finite,lazaric2012finite, scherrer2015approximate, geist2019theory}. 
In these RL problems, such formulations, while being nominally different, coincide with ours in (\ref{ampi})
as the greedy policy $\mathcal{T}^{\mathcal{G}(v^k)} v^k = \max_{\pi} \mathcal{T}^{\pi} v^k$ is optimal for the policy improvement step. 
Inspired by the literature, we first assume two oracles to execute these two steps for now, leading to the following assumption:

\begin{assumption}
     \label{assum_policy_realizability}
     There exist oracles that approximately solve (i) the policy improvement step with errors $\{\epsilon_k\}$, and (ii) the policy evaluation step in AMPI  (\ref{ampi}) with error $\{\epsilon'_k\}$.
\end{assumption}

Assumption \ref{assum_policy_realizability} also requires the existence of an $\epsilon'$-approximated policy for policy improvement, which is not obviously valid as $\mathcal{T}^*$ may correspond to different best policies for different states. This issue is resolved if the greedy policy is optimal (as in Section \ref{sec:contamination}) and is further discussed in Section \ref{sec:inherent}.  We also provide a standalone solution for this issue in Appendix \ref{appendix:1}.

Now, under Assumption \ref{assum_policy_realizability}, for any $k$-th iteration, we are ready to control the loss $l_k \coloneqq v^{\pi_t} - v^{\pi^k}$ via AMPI, where $\pi_t$ is the solution to\BestResponse~problem with respect to $\lambda_t$, $v^{\pi}$ represents the unique stationary point of $\mathcal{T}^{\pi}$ for any policy $\pi$ whose uniqueness is guaranteed by contraction in Proposition \ref{basic_properties_T}. The analysis is analogous to that in \citet{scherrer2015approximate}.

\begin{theorem} \label{br_L_inf_error_bound}
    Under Assumption \ref{assum_policy_realizability}, applying (\ref{ampi}) for $k$-th iterations, the loss $l_k$ satisfy, \begin{equation}
        l_k \le O(\gamma^k) + {(2\bar{\epsilon}(\gamma-\gamma^k) + \bar{\epsilon}' (1-\gamma^k))}/{(1-\gamma)^2} \xrightarrow{k \to \infty} {(2\bar{\epsilon}\gamma+\bar{\epsilon}')}/{(1-\gamma)^2},
    \end{equation} 
    where $\bar{\epsilon} \in \mathbb{R}^{\St}$ is the upperbound of errors $\{\epsilon_k\}$, i.e. $\forall k,\epsilon_k \le \bar{\epsilon}$, and $\bar{\epsilon}' \in \mathbb{R}^{\St}$ is similarly defined as the upper bound of $\{\epsilon'_k\}$. 
\end{theorem}

Theorem \ref{br_L_inf_error_bound} shows that, when errors are relatively small, our AMPI (\ref{ampi}) guarantees convergence to the solution of\BestResponse~under Assumption \ref{assum_policy_realizability}. Combining this with a no-regret online algorithm, we complete the general framework for the DRC-RL problem as in Algorithm \ref{Meta_Algo}.

Finally, while the oracle for the approximate policy evaluation step can be implemented for several popular uncertainty sets \citep{shi2023curious, clavier2023towards}, the greedy policy solution for the approximate policy improvement step does not work, at least for our consistency operator $\mathcal{T}^{\pi}$. This is due to its dependency on the whole policy in its definition (\ref{general_operator}).  
Moreover, there is currently no provable efficient instantiation for general uncertainty set to enable the approximate policy improvement step in DRC-RL problems. Given the fact that\BestResponse~corresponds to the fundamental dual function of (\ref{conservative_form}), and that the policy improvement step consists of a popular class of iteration methods for\BestResponse~type problems, we wonder: \begin{itemize}[nosep]
    \item[(Q1)] Can we design a specific uncertainty set for our operator $\mathcal{T}^{\pi}$ that enables solving DRC-RL without oracles, e.g. using greedy policies?
    \item[(Q2)] Is it possible to design a better consistency operator that makes greedy policies optimal that in turn provably solves DRC-RL with our framework?
\end{itemize} 
We address these two questions in the next two sections, respectively.

%% file: 04-contamination.tex
\section{DRC-RL with R-Contamination Uncertainty Sets}
\label{sec:contamination}

In this section, we address (Q1) via a focus on the R-contamination uncertainty sets $\mathcal{P}_{s, a} \coloneqq \{(1 - \beta)P^o_{s, a} + \beta q \,|\, q \in \triangle(\mathcal{S})\}$ with a scalar robust level $\beta \in \mathbb{R}$. This uncertainty set has been studied in distributionally robust RL recently \citep{wang2022robust, li2023first}. Considering this,  we can simplify our consistency operator $\mathcal{T}^{\pi}$ without any loss in solving\BestResponse~as \begin{align}
    [\mathcal{T}^{\pi} v](s) =&~ (r - \lambda_t^{\top}g)(s, \pi(s)) + \gamma (1 - \beta) \langle P^o_{s, \pi(s)}, v\rangle + \gamma \beta (\min_{s'} V^{\pi}_r(s') - \lambda_t^{\top} \min_{s'}V^{\pi}_{g}(s')). \label{contamination_operator}
\end{align}
Please refer to Appendix \ref{sec:derivation_contamination} for detailed proof. Additionally, we adopt the following fail-state assumption \citep{panaganti2022robust}. \begin{assumption}[fail-state]
    There is a fail state $s_f$ for all the RMDPs, such that $r(s_f, a) = 0, g_i(s_f, a) = 0$ and $P_{s_f, a}(s_f) = 1$, for all $a \in \mathcal{A}$ and $P \in \mathcal{P}$. 
\end{assumption}
The fail-state assumption is commonly satisfied in practice as it corresponds to an end-game state in the simulator or real-world systems, in which all constraints are violated and the reward is zero. Under this fail-state assumption, we always have $\min_{s'} V^{\pi}_r(s') =  \min_{s'}V^{\pi}_{g}(s') = 0$, which makes the operator $\mathcal{T}^{\pi}$ correspond to a standard bellman consistency operator with shortened discount factor, 
\begin{equation}
    [\mathcal{T}^{\pi} v](s) = (r - \lambda_t^{\top}g)(s, \pi(s)) + \gamma (1 - \beta) \langle P^0_{s, \pi(s)}, v\rangle .\label{contamination_operator_simplified}
\end{equation}
Given that $\mathcal{T}^{\pi}$ takes the form of a standard consistency operator, the greedy policy is available for the policy improvement step. Thus, any instantiation of AMPI (\ref{ampi}), such as simple value iteration or policy iteration, can efficiently and provably solve\BestResponse~problem under small errors. The distributionally robust constrained problem therefore has a provable solution in the case of the R-contamination uncertainty set. The details of an instantiation of\BestResponse~and\OnlineAlgo~are discussed in Appendix \ref{appendix:2}.

Our result for the R-contamination uncertainty set indicates that a smaller discount factor, i.e. a smaller effective horizon, gives higher distributional robustness. However, as this smaller discount factor is a consequence of the robustness objective, the discount factor should remain unchanged, when designing constraints threshold $\tau$ and testing in shifted environments. 
Such an unchanged discount factor in thresholds and tests differs our solution from simply scaling the problem.


Finally, it is worth noting that the case of no constraints, e.g. $g = \tau = 0$, implies that our analysis also gives a provable solution to the distributionally robust RL problem on R-contamination uncertainty sets.

%% file: 05-general.tex
\section{On the Intractability of Greedy Policies for DRC-RL}
\label{sec:inherent}

In this section, we answer question (Q2) with a negative result showing that the combination of constraints and distributional robustness requires different algorithmic tools than either robust RL or constrained RL do alone. 
In detail, we show that for any `good' consistency operator, the greedy policies are not generally optimal for the policy improvement step in (\ref{ampi}). This, in turn, prevents any algorithm from a popular class of iteration methods from tractably solving the DRC-RL problem. 

To begin, we formally define the optimality of greedy policy and connect it to the operator. We assume policy $\pi \in \RSA$ includes both deterministic and stochastic policies.

\begin{definition}[Greedy Policy Enabling]
    We state an consistency operator $\mathcal{T}^{\pi}$ enables the greedy policy if there exist a function $g : \mathbb{R}^{\mathcal{S}} \times \mathcal{A} \to \mathbb{R}^{\mathcal{S}}$ such that $\forall v \in \mathbb{R}^{\mathcal{S}}, s \in \mathcal{S}, \max_{a \in \mathcal{A}} g(v, a)[s] = \max_{\pi \in \Pi} [\mathcal{T}^{\pi}v](s)$, i.e. greedy policies are optimal.
    \label{greedy_policy}
\end{definition}
\begin{definition}[Operator Linearity]
    The consistency operator $\mathcal{T}^{\pi}$, that takes policy $\pi$ as an input, is linear if there exists a function $f : \RSt \to \RSA$ that is independent of $\pi$, such that $\forall v \in \RSt, s \in \mathcal{S}$, we have $[\mathcal{T}^{\pi} v](s) = \langle \pi[s, \cdot], f(v)[s, \cdot]\rangle = \langle \pi, f(v) \rangle_s$. 
    \label{linearity_definition}
\end{definition}

Take distributional robust RL as an example, its robust consistency operator (\ref{dr_operator}): $[\mathcal{T}^{\pi}_{r, \mathrm{rob}} v](s) = \langle \pi[s, \cdot], f(v)[s, \cdot]\rangle$ is linear where $f(v)[s, a] = r_{s, a} + \min_{P \in \mathcal{P}} \langle P_{s, a}, v \rangle$. Thus, one only needs to treat $f(v)[s, a]$ as $g(v, a)[s]$ in Definition \ref{greedy_policy} to enable the greedy policy. 
It is not hard to find that a linear operator naturally enables the greedy policy as in the above example, and we generalize this in the following result. All proofs for this section are presented in Appendix \ref{appendix:proof_5}.

\begin{lemma}
    \label{eq_linear_greedy}
    The linear operator is equivalent to greedy policy enabled operators in Definition \ref{greedy_policy} in the following ways: 
    (i) If a consistency operator is a linear operator, then it enables the greedy policy.
    (ii) If a consistency operator $\mathcal{T}^{\pi}$ enables the greedy policy, then there always exists a linear operator $\mathcal{T}^{\pi}_{linear}$ that can substitute $\mathcal{T}^{\pi}$ without any loss in policy improvement step, i.e. $\forall v \in \RSt, s \in \St, \max_{\pi \in \Pi} [\mathcal{T}^{\pi}v](s) = \max_{\pi \in \Pi} [\mathcal{T}^{\pi}_{linear}v](s)$.
\end{lemma}

However, the following shows that it is impossible to have a consistency operator that simultaneously converges as a contraction to our target in\BestResponse~and enables the greedy policy.
\begin{theorem}
    \label{impossible_thm}
    There is no consistency operator $\mathcal{T}^{\pi}$, that takes any policy $\pi$ as an input, simultaneously satisfies for given $\gamma$ and every $\lambda_t$: (i) Linearity. (ii) $\gamma$-Contraction to our target: the operator $\mathcal{T}^{\pi}$ is a contraction such that for every policy $\pi \in \Pi$, $V^{\pi}_{r} - \lambda_t^{\top}V^{\pi}_{g}$ is the unique stationary point of  $\mathcal{T}^{\pi}$.
\end{theorem}
\begin{corollary}
    \label{impossible_cor}
    There is no consistency operator $\mathcal{T}^{\pi}$ that enables greedy policy while retaining as a $\gamma$-contraction to $V^{\pi}_{r} - \lambda_t^{\top}V^{\pi}_{g}$.
\end{corollary}

Theorem \ref{impossible_thm} and Corollary \ref{impossible_cor} highlight the additional difficulty of DRC-RL as compared to robust or constrained RL, where iterative methods can be successful. Our proof (Appendix \ref{proof:inherent}) demonstrates how this difficulty arises from the combination of constraints and distributional robustness.

Comparing our previous success in the case of R-contamination sets (Section \ref{sec:contamination}) to these impossible results, it becomes evident that the additional fail-state assumption is critical.
This assumption restricts the possible value function space and provides an additional transition structure that avoids the challenges underlying Theorem \ref{impossible_thm}.

Following the fail-state assumption, we believe it is possible to design additional conditions for other forms of uncertainty sets that resolve such impossibilities. 
Although we have not yet found clear and rigorous guidance, we tentatively acknowledge that the absence of linearity in worst-case transition kernels is essential for Theorem \ref{impossible_thm}. 
Therefore a structured value function space or an augmented state space (such as in \citet{sootla2022saute}) might be helpful to give tractable solutions with our framework in Section \ref{sec:framework}.



%% file: 06-conclusion.tex
\section{Experiments and Evaluation}
\label{sec:experiments}
In this section, we present a focused experiment to validate our solution in Section \ref{sec:contamination}. This solution specializes to the case of constrained RL in \citet{le2019batch} when setting the robustness level to 0. We present critical settings here and refer to Appendix \ref{appendix:3} for more details.

\begin{figure}[t]
    \centering
    \subfigure[Constraints Satisfaction under Shifts]{
        \includegraphics[width = .56\linewidth]{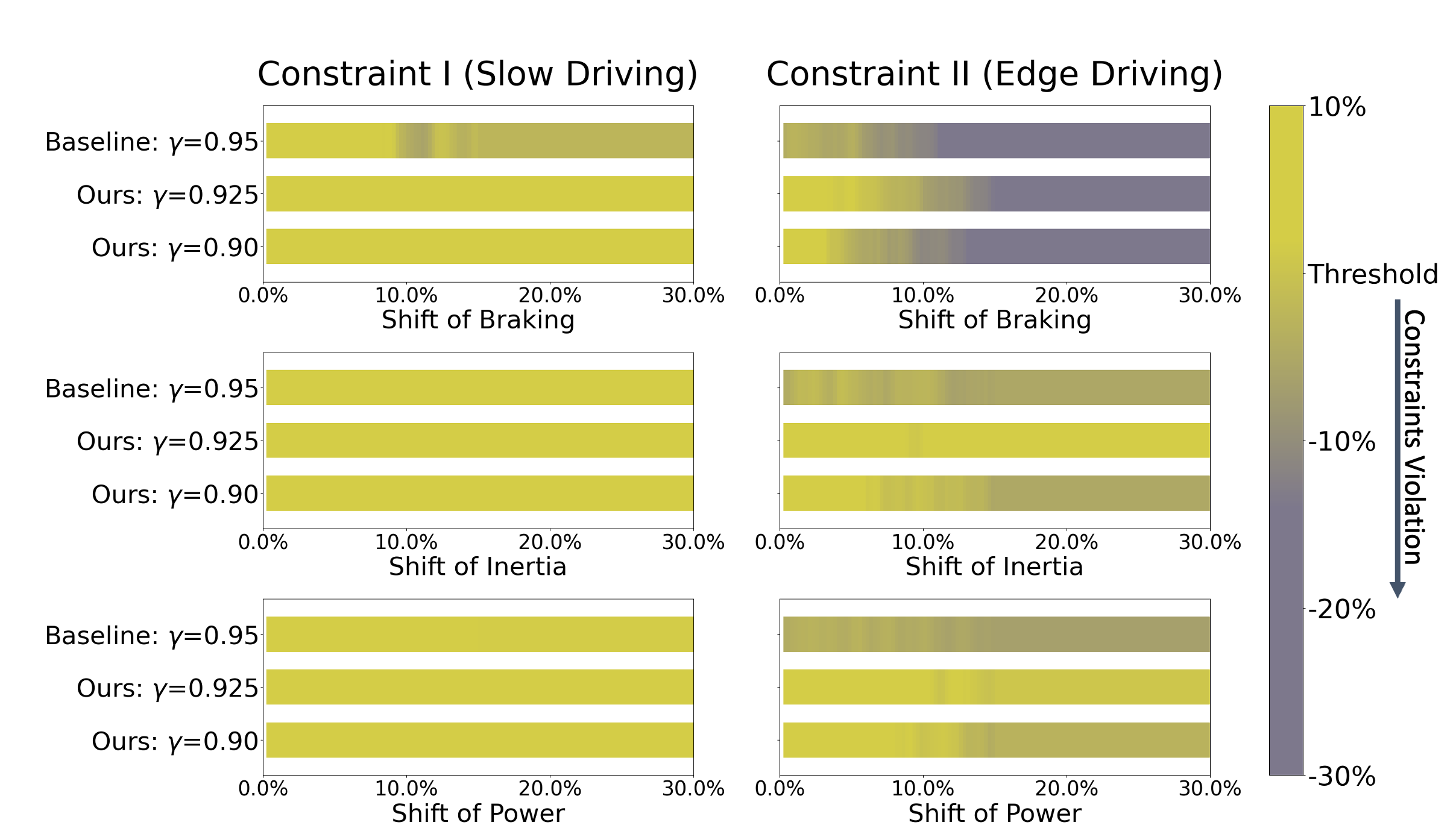}
    }
    \subfigure[Performance under Steering Shift]{
        \includegraphics[width = .40\linewidth]{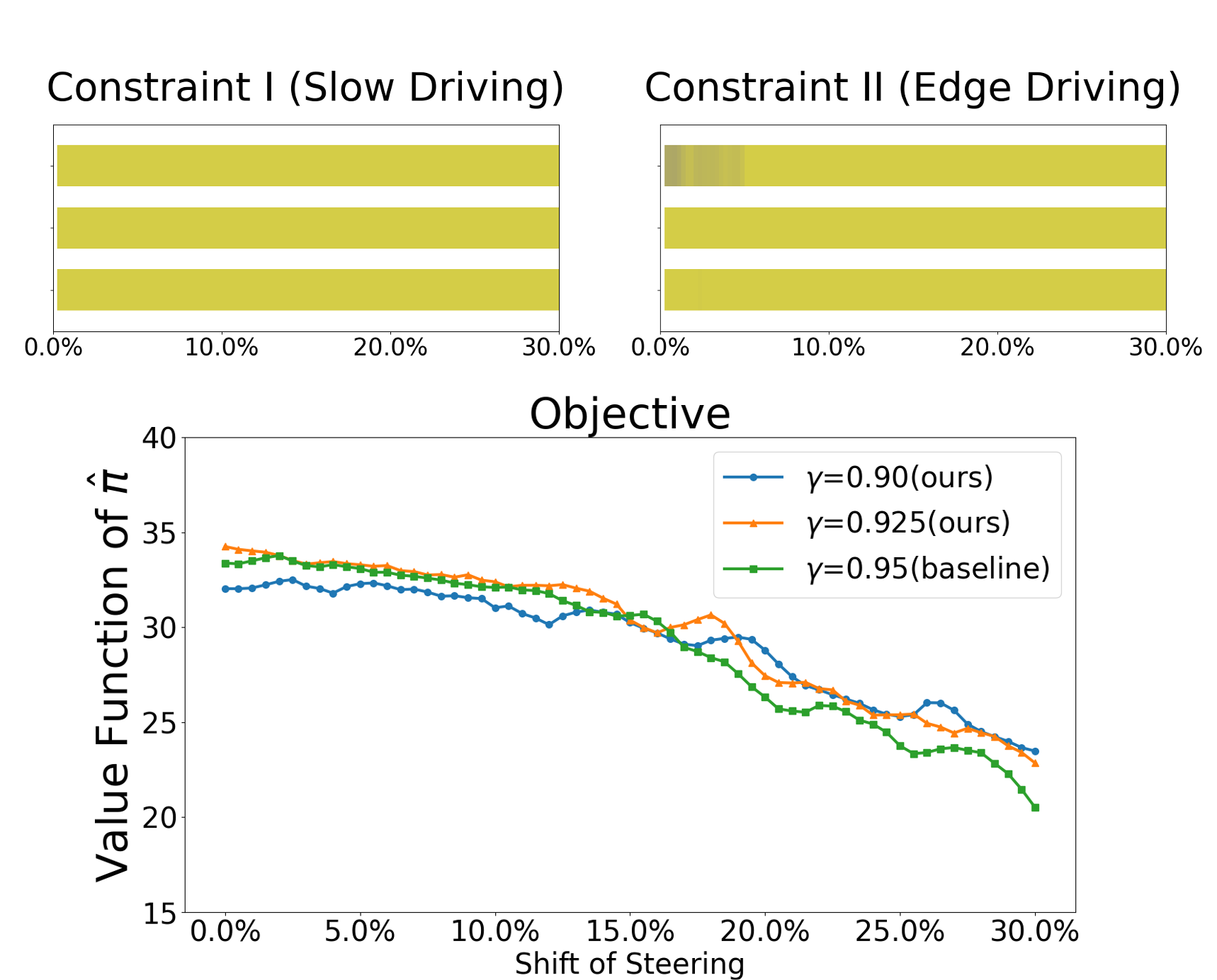}
    }
    \vspace{-0.15in}
    \caption{The four bar graphs denote the constraints satisfaction (green means satisfied) when shifts of power, inertia, braking magnitude, and steering angle occur. The lower right figure indicates the value of the objective (higher is better) when the steering angle is shifted. All evaluations are based on the value function (accumulated rewards) of mixture policy $\hat{\pi}$.} \label{fig:experiments}
\end{figure}

\textbf{Task Setting.} We choose the high-dimensional Car Racing task \citep{towers_gymnasium_2023} where the agents must traverse as far as possible on track, with each reward for passed tile of track and a small negative reward for each second. Two constraints are designed: slow driving and edge driving. 
States off the track excessively are considered fail-states.


\textbf{Algorithms Design and Baseline.} We adopt Fitted Q Iteration \citep{ernst2005tree} and Exponentiated Gradient \citep{kivinen1997exponentiated} for\BestResponse~and\OnlineAlgo. Evaluations are executed using the simulator and meet the requirements of Proposition \ref{meta_prop}. We choose $\gamma = 0.95$ as the initial discount factor for the baseline, since in this case our solution retrieves the constrained RL algorithm in \citet{le2019batch} and is considered as zero robustness level . 
We select two robustness levels $\beta$ for evaluation.  These result in discount factors $\gamma(1 - \beta) = \{0.90, 0.925\}$, and correspond to a maximized $50\%$ decrease of the effective horizon. All other hyperparameters, including random seeds, do not change across $\gamma$.

\textbf{Procedure and Criteria.} For each $\gamma$, 25 rounds are executed, which guarantees a duality gap of less than $.01$. The mixture policy $\hat{\pi} = \sum_{t = 1}^{25} \frac{1}{25} \pi_t$ is then tested in the shifted environments. When testing with a shifted environment, the objective and two constraints for each mixture policy are measured with value function and initial discount factor $\gamma = 0.95$. All results in shifted environments are smoothed as the mean of 3 random seeds and a $1\%$ shift window. 
We regard constraints as hard and prioritize constraint satisfaction as the main criterion for robustness, with the value of the objective as the secondary criterion. 

\textbf{Results.} Our experiments verify that smaller $\gamma$ learns a more robust policy in the car-racing example. 
Results are shown in Figure \ref{fig:experiments}: In the left plot, we present constraint satisfaction under shifts of power, inertia, and brake magnitude.  Our learned policies with smaller discount factors perform better.  In the right plot, we present a full set of evaluations when the steering angle is shifted.  Here the learned policies with smaller discount factors not only satisfy more constraints, but may also have better objective value when a certain shift occurs.

\section{Related Works}\label{sec:related-work}


\textbf{Constrained RL.} 
Constrained RL aims at maximizing expected cumulative reward while adhering to specified constraints. Applications of Constrained RL cover a wide array of topics, such as resource allocation for numerous users in grid systems \citep{wang2020safeoff,de2021constrained,mo2023framework},
human satisfaction in human-robot interaction \citep{el2020towards,liu2023safe}, and the safety level of robotic agents \citep{wachi2020safe,zhang2020cautious,brunke2022safe,tambon2022certify}. 
The underlying decision-making problem in constrained RL can be represented as a constrained Markov decision process \citep{altman2021constrained} that has a bilevel structure with strong duality \citep{miryoosefi2019reinforcement, paternain2019constrained}. 
Most works in constrained RL employ model-based methods \citep{efroni2020exploration, bura2022dope}. To develop model-free and policy gradient methods, many additional constrained RL algorithms, involving online \citep{ding2021provably,wachi2021safe} or offline \citep{le2019batch} interactions, embrace a primal-dual methodology. 

\textbf{Distributionally Robust RL.}
Distributionally robust RL tackles the challenge of formulating a policy resilient to shifts between training and testing environments by using robust Markov decision process \citep{nilim2003robustness,iyengar2005robust} as the underlying decision-making problem. Many heuristic works \citep{xu2010distributionally, wiesemann2013robust, yu2015distributionally, mannor2016robust, russel2019beyond} have shown distributional robustness is valuable when environment shifts occur. Recently, work has also started to provide a concrete theoretical understanding of sample complexity \citep{yang-2022, pmlr-v151-panaganti22a, xu-panaganti-2023samplecomplexity,shi2024sample,shi2022distributionally}. Additionally, other recent work \citep{panaganti2020robust,wang2021online,pmlr-v151-panaganti22a} has employed general function approximation to devise model-free online and offline robust RL algorithms, though without theoretical guarantees.

\textbf{Related work on DRC-RL.} The study of DRC-RL can be traced back to at least \citet{russel2020robust} and \citet{mankowitz2020robust}, where the basic formulation is proposed and first-order methods, such as Robust Constrained Policy Gradient (RCPG), are studied. More recently, \citet{wang2022robust} focuses on the dual problem and proposes a first-order method that achieves the convergence guarantee to a stationary point with additional approximation, and \citet{bossens2023robust} extends RCPG to Lagrangian or adversarial updates for the dual problem of a different formulation.
A simultaneous work of ours is \citet{sun2024constrained}, where they proposed a projected gradient descend style algorithm that guarantees per-step improvement and constraints violation.
However, these existing works do not provide provable guarantees for an end-to-end framework. 

Beyond the distributional robustness we studied, a variety of relaxations quantifying robustness, such as assuming a stochastic setting with uncertainty on the distribution of kernels \citep{queeney2024risk} or considering risk-averse constraints \citep{kim2024trust}, are studied. Such relaxation provides a more informative structure to avoid the minimax formulation, and is therefore beyond the scope of this paper.




\section{Conclusion}

In this paper, we present an algorithmic framework for the distributionally robust constrained RL problem (DRC-RL).
Our framework provides the first efficient provable solution for R-contamination uncertainty sets.  We additionally prove the intractability of greedy policies for general uncertainty sets, which prevents the use of popular iterative methods unless the uncertainty sets and additional assumptions maintain additional structures. 
In the case of R-contamination uncertainty sets, a simple rule relating the discount factor and distributional robustness is discovered, which may be of broader interest. 
In particular, such future questions may involve inventing new algorithms for other uncertainty sets, exploring the tractability of gradient-based methods, or studying the additional structure we need for encompassing both distributional robustness and constraint satisfaction.

\section{Acknowledgment}

The work is supported by the NSF through CNS-2146814, CPS-2136197, CNS-2106403, NGSDI-2105648, CCF-1918865, and by the Resnick Sustainability Institute at Caltech, and by the gift from Latitude AI.
Z. Zhang acknowledges support from the Tsien's Excellence in Engineering Program at Tsinghua. 
K. Panaganti acknowledges support from the 'PIMCO Postdoctoral Fellow in Data Science' fellowship at Caltech.
The work of L. Shi is supported in part by the Resnick Institute and Computing, Data, and Society Postdoctoral Fellowship at Caltech. 
The authors thank anonymous RLC 2024 reviewers for their constructive comments on an earlier draft of this paper. 

%% file: 07-appendix.tex
\newpage
\appendix

\section{Proof for Section~\ref{sec:framework}: DRC-RL with General Uncertainty Sets}
\label{sec:appendix_proof_sec3}

\subsection{Proof of Proposition \ref{policy_convexify}}
\label{sec:proof_convexify}

We assume $\pi_{\alpha,\{\pi_i\}_{i=1}^T}, \pi'_{\alpha', \{\pi_i\}_{i=1}^T} \in Conv(\Pi)$ are two mixed policies defined in Section \ref{sec:meta} with the same candidate deterministic policy set without loss of generality (Otherwise we could simply combine two sets and set zero to new candidate policy for each categorical distribution). For brevity, we use $\pi_{\alpha}$ and $\pi_{\alpha'}$ to denote them.

First, by the definition of mixed policy in Section \ref{sec:meta}, $Conv(\Pi)$ is indeed a convex hull that linear combination of policies over $c\in[0, 1]$ satisfies, \begin{equation}
    \pi_{c\alpha + (1 - c)\alpha'} = c\pi_{\alpha} + (1 - c)\pi_{\alpha'} \in Conv(\Pi).
\end{equation}

Then we show that the robust value function is linear to mixed policy, with $c \in [0, 1]$, \begin{align}
    V^{\pi_{c\alpha + (1 - c)\alpha'}}_r =&~ \sum_{\pi} (c\alpha(\pi) + (1 - c)\alpha'(\pi)) V^{\pi}_r\\
    =&~c\sum_{\pi} \alpha(\pi) V^{\pi}_{r} + (1 - c)\sum_{\pi} \alpha'(\pi) V^{\pi}_r \\
    =&~ cV^{\pi_{\alpha}} + (1 - c)V^{\pi_{\alpha'}}.
\end{align} where the first equality comes from the definition of the robust value function of a mixed policy in Section \ref{sec:meta}.

Naturally, the robust value function for rewards of constraints $V^{\pi}_g$ is also linear to mixed policy. 

Therefore, the constrained optimization problem (\ref{conservative_form}) becomes convex, and hence the strong duality holds with Slater's condition \citep{boyd2004convex}.

\subsection{Proof of Proposition \ref{meta_prop}}

First, as proved in \ref{sec:proof_convexify}, Lagrangian $L(\pi, \lambda) = V^{\pi}_{r} - \lambda^T V^{\pi}_g$ is linear to both policy $\pi$ and multiplier $\lambda$. (We are treating $\Pi$ as its convex hull $Conv(\Pi)$ now.)

When\OnlineAlgo~is chosen as a no-regret online learning algorithm with the negative Lagrangian $-L(\pi, \lambda)$ as loss \citep{kivinen1997exponentiated, zinkevich2003online}, we have \begin{equation}
    \sum_{t} (-L)(\pi_t, \lambda_t) \ge \max_{\lambda} \sum_{t} (-L)(\pi_t, \lambda) - o(T) \label{eq_no_regret}
\end{equation}
Then, recalling that $\pi_t$ is the\BestResponse~(\ref{best_response}) given current $\lambda_t$, we have 
\begin{align}
    \min_{\lambda} L(\hat{\pi}_T, \lambda) =&~ \min_{\lambda} \frac{1}{T} \sum_{t} L(\pi_t, \lambda)\\
    =&~-\max_{\lambda} \frac{1}{T} \sum_{t} - L(\pi_t, \lambda)\\
    \overset{\mathrm{(i)}}{\ge}&~ \frac{1}{T} \sum_t L(\pi_t, \lambda_t) - \frac{o(T)}{T}\\
    \overset{\mathrm{(ii)}}{\ge}&~ \frac{1}{T} \sum_t L(\pi, \lambda_t) - \frac{o(T)}{T}~~,~\forall \pi \in \Pi\\
    =&~ L(\pi, \hat{\lambda}_T) - \frac{o(T)}{T}~~,~\forall \pi \in \Pi
\end{align}
where the first and the last equalities come from the linearity of the Lagrangian w.r.t. the policy and the Lagrange multiplier, and the definition of $\hat{\pi}_T$ and $\hat{\lambda}_T$ in meta algorithm \ref{Meta_Algo}; (i) holds by Eq.(\ref{eq_no_regret}), and (ii) arises from the fact that $\pi_t$ is the\BestResponse~(\ref{best_response}) 

Additionally, inserting the fact $\max_{\pi} L(\pi, \hat{\lambda}_T) \ge L(\hat{\pi}_T, \hat{\lambda}_T) \ge \min_\lambda L(\hat{\pi}_T, \lambda)$ 
one has
\begin{equation}
    \max_{\pi} L(\pi, \hat{\lambda}_T) \ge \min_\lambda L(\hat{\pi}_T, \lambda) \ge \max_{\pi} L(\pi, \hat{\lambda}_T) - \frac{o(T)}{T}.
\end{equation} 
Finally, recalling that $L_{\mathrm{max}} = \max_{\pi} L(\pi, \hat{\lambda}_T)$ and  $L_{\mathrm{min}} = \min_\lambda L(\hat{\pi}_T, \lambda)$ in meta algorithm \ref{Meta_Algo}. The duality gap $L_{\mathrm{max}} - L_{\mathrm{min}}$ is bound to smaller than positive threshold $\omega$. 
The exact convergence rate of Algorithm \ref{Meta_Algo} will depend on the choice of\OnlineAlgo. For example, the algorithm will terminate after $\mathcal{O}(\frac{1}{\omega^2})$ rounds if online algorithms with regret scaling as $\Omega(\sqrt{T})$ are chosen (such as online gradient descent with regularizer).   
\qed

\subsection{Proof for Proposition \ref{basic_properties_T}}

Through out this proof, we denote $P^o_{\pi}$ as the vector $[P^o_{s, \pi(s)}]^{\top}_{s \in \mathcal{S}}$, and $v_1, v_2 \in \RSt, s \in \cS$.

\textbf{Monotonicity}. By the definition of $\mathcal{T}^{\pi}$ , we have with $v_1 \ge v_2$, \begin{equation}
    \mathcal{T}^{\pi} v_1 - \mathcal{T}^{\pi} v_2 = \gamma \langle P^o_{\pi}, v_1 - v_2 \rangle \ge 0.
\end{equation}
Then for $\mathcal{T}^{*}$, we denote $\pi^s_1 \coloneqq \argmax_{\pi \in \Pi}  \mathcal{T}^{\pi} v_1(s)$ and similar for $\pi^s_2$, we then have for every state $s$, \begin{equation}
    \mathcal{T}^{*} v_1(s) - \mathcal{T}^{*} v_2(s) = \mathcal{T}^{\pi^s_1}_s v_1(s) - \mathcal{T}^{\pi^s_2}_s v_2(s) \ge \mathcal{T}^{\pi^s_1}_s v_1(s) - \mathcal{T}^{\pi^s_1}_s v_2(s) \ge 0.
\end{equation}
The last inequality comes from the monotonicity of $\mathcal{T}^{\pi}$, which completes the proof.

\textbf{Transition Invariance}. From the definition of $\mathcal{T}^{\pi}$ in (\ref{general_operator}), we have \begin{equation}
    \mathcal{T}^{\pi} (v_1 + c\mathbf{1}) = \mathcal{T}^{\pi} v_1 + \gamma \langle P^o_{\pi}, c\mathbf{1} \rangle = \mathcal{T}^{\pi} v_1 + \gamma c\mathbf{1}.
\end{equation}
Then similar for $\mathcal{T}^{*}$, we have, \begin{equation}
    \mathcal{T}^{*} (v_1 + c\mathbf{1}) = \mathcal{T}^{*} v_1 + \gamma \langle P^o_{\pi}, c\mathbf{1} \rangle = \mathcal{T}^{*} v_1 + \gamma c\mathbf{1}.
\end{equation}

\textbf{Contraction}. we first show the $\gamma$-contraction property of $\mathcal{T}^{\pi}$ by its definition in (\ref{general_operator}), \begin{equation}
    |[\mathcal{T}^{\pi} v_1](s) - [\mathcal{T}^{\pi} v_2](s)| = |\gamma \langle P^0_{s, \pi(s)}, v_1 - v_2 \rangle|
    \le \gamma \Vert v_1 - v_2\Vert_{\infty}. \label{temp3}
\end{equation} where the last inequality comes from the distribution nature of $P^0_{s, \pi(s)}$.

Then for $\mathcal{T}^{*}$ and state $s$, we assume $v_1(s) \ge v_2(s)$ without losing generality, \begin{align}
     |[\mathcal{T}^{*} v_1](s) - [\mathcal{T}^{*} v_2)](s)| =&~ [\mathcal{T}^{\pi^s_1}_s] v_1(s) - [\mathcal{T}^{\pi^s_2}_s] v_2(s)  \notag\\
     \overset{\mathrm{(i)}}{\le}&~ 
     [\mathcal{T}^{\pi^s_1}_s v_1](s) - [\mathcal{T}^{\pi^s_1}_s v_2](s) \\
     \le&~ 
     |[\mathcal{T}^{\pi^s_1}_s v_1](s) - [\mathcal{T}^{\pi^s_1}_s v_2](s)| \\
     \overset{\mathrm{(ii)}}{\le}&~ \gamma \Vert v_1 - v_2\Vert_{\infty}
\end{align} where the first equality comes from the definitions of $\pi^s_1$ and $\pi^s_2$, inequality (i) comes from the fact that $[\mathcal{T}^{\pi^s_2}_s] v_2(s) = [\max_{\pi} \mathcal{T}^{\pi} v_2](s) \ge [\mathcal{T}^{\pi^s_1}_s v_2](s)$, and inequality (ii) comes from the contraction property of $\mathcal{T}^{\pi}$ in (\ref{temp3}).

Finally, We apply $\mathcal{T}^{\pi}$ on our objective $V^{\pi}_r - \lambda_t^{\top} V^{\pi}_g$, which yields \begin{align}
    [\mathcal{T}^{\pi} (V^{\pi}_r - \lambda_t^{\top} V^{\pi}_g)](s) =&~ (r - \lambda_t^{\top}g)(s, \pi(s)) + \gamma \langle P^o_{s, \pi(s)}, V^{\pi}_r - \lambda_t^{\top} V^{\pi}_g) \rangle\\
    &~+ \gamma \min_{P \in \mathcal{P}} \langle P_{s, \pi(s)} - P^o_{s, \pi(s)}, V^{\pi}_{r} \rangle - \gamma \lambda_t^{\top} \min_{P \in \mathcal{P}} \langle P_{s, \pi(s)} - P^o_{s, \pi(s)}, V^{\pi}_{g} \rangle  \nonumber \\
    =&~r(s, \pi(s)) +  \gamma \langle P^o_{s, \pi(s)}, V^{\pi}_r \rangle + \gamma \min_{P \in \mathcal{P}} \langle P_{s, \pi(s)} - P^o_{s, \pi(s)}, V^{\pi}_{r} \rangle \\
    &~- \lambda_t^{\top} \left(g(s, \pi(s)) +  \gamma \langle P^o_{s, \pi(s)}, V^{\pi}_g \rangle + \gamma \min_{P \in \mathcal{P}} \langle P_{s, \pi(s)} - P^o_{s, \pi(s)}, V^{\pi}_{g}) \rangle \right)\nonumber \\
    =&~r(s, \pi(s)) + \gamma \min_{P \in \mathcal{P}} \langle P_{s, \pi(s)}, V^{\pi}_{r} \rangle \\
    &~- \lambda_t^{\top} \left(g(s, \pi(s)) +  \gamma \min_{P \in \mathcal{P}} \langle P_{s, \pi(s)}, V^{\pi}_{g}) \rangle \right)\nonumber \\
    =&~ [\mathcal{T}^{\pi}_{r, rob} V^{\pi}_r](s) - \lambda_t^{\top} [\mathcal{T}^{\pi}_{g, rob} V^{\pi}_g](s)\\
    =&~ (V^{\pi}_r - \lambda_t^{\top} V^{\pi}_g)(s)
\end{align} where the first equality comes from the definition of consistency operator $\mathcal{T}^{\pi}$ in (\ref{general_operator}), the second equality is a simple rearrangement, the third equality comes from the fact that $\langle P^o_{s, \pi(s)}, V^{\pi}_{g} \rangle$ is independent of $P \in \mathcal{P}$, the fourth equality comes from the definition of robust consistency operator in \ref{dr_operator}, and the last equality comes from the contraction property of robust consistency operator \citep{iyengar2005robust}.

Combine with the fact that operator $\mathcal{T}^{\pi}$ is a $\gamma$-contraction, $V^{\pi}_r - \lambda_t^{\top} V^{\pi}_g$ is the only stationary point of $\mathcal{T}^{\pi}$, the proof is complete.
\qed

\subsection{Proof for Theorem \ref{br_L_inf_error_bound}}

\subsubsection{Preliminary}

To start with, we introduce two preliminary propositions from \citet{scherrer2015approximate} where Proposition \ref{tech_l1} builds three relations, and Proposition \ref{tech_inequality} is a direct application of these three relations. Definition \ref{def_discounted_kernel_set} is used to simplify notation in proposition \ref{tech_inequality}.

\begin{proposition}[Lemma 2, \citet{scherrer2015approximate}]
    Consider approximate modified policy iteration scheme with standard bellman operator $\mathcal{T}_{\mathrm{st}}^{\pi}$. \begin{equation}
        \begin{cases}
            \pi^{k + 1} &=~~ \argmax^{\epsilon'_{k + 1}}_{\pi} \mathcal{T}_{\mathrm{st}}^{\pi} v^k = \mathcal{G}^{\epsilon'_{k + 1}} v^k\\
            v^{k + 1} &=~~ (\mathcal{T}_{\mathrm{st}}^{\pi^{k + 1}})^m v^k + \epsilon_{k + 1} \label{ampi_standard}
        \end{cases} 
    \end{equation} where $\pi^{k + 1}$ is the greedy policy with respect to $v^k$ with some error $\epsilon'_{k + 1}$, e.g. $\forall \pi \in \Pi, \mathcal{T}_{\mathrm{st}}^{\pi^{k + 1}} v^k + \epsilon'_{k + 1} \ge \mathcal{T}_{\mathrm{st}}^{\pi} v^k$.

    Let $v_*$ denote the optimal value function, $d_k \coloneqq v_* - (\mathcal{T}_{\mathrm{st}}^{\pi^{k}})^m v^{k - 1}$, $s_k \coloneqq (\mathcal{T}_{\mathrm{st}}^{\pi^{k}})^m v^{k - 1} - v_{\pi^k}$ and $b_k \coloneqq v^k - \mathcal{T}_{\mathrm{st}}^{\pi^{k + 1}} v^{k}$, then for $k \ge 1$, we have, \begin{align}
        b_k \le&~ (\gamma P_{\pi^k})^m b_{k - 1} + x_k\\
        d_{k + 1} \le&~ \gamma P_{\pi_*}d_k + y_k + \sum_{j = 1}^{m - 1}(\gamma P_{\pi^{k + 1}})^j b_k\\
        s_k =&~ (\gamma P_{\pi^k})^m(I - \gamma P_{\pi^k})^{-1}b_{k - 1}
    \end{align} where $x_k \coloneqq (I - \gamma P_{\pi^k})\epsilon_k + \epsilon'_{k + 1}$ and $y_k \coloneqq -\gamma P_{\pi_*}\epsilon_k + \epsilon'_{k + 1}$. \label{tech_inequality}
\end{proposition}

\begin{definition}[$\Gamma$-matrix \citep{scherrer2015approximate}]
    For a positive integer $n$, we define $\mathbb{P}_n$ as the smallest set of discounted transition kernels that are defined as (i) for any set of $n$ policies $\{\pi^i\}$, $(\gamma P^o_{\pi^1})(\gamma P^o_{\pi^2})\cdots(\gamma P^o_{\pi^n}) \in \mathbb{P}_n$, where $P^o_{\pi}(s, s') = P^o_{s, \pi(s)}(s')$. (ii) for any $\alpha \in (0, 1)$ and $(P_1, P_2) \in \mathbb{P}_n \times \mathbb{P}_n$, $\alpha P_1 + (1 - \alpha)P_2 \in \mathbb{P}_n$. With slight abuse of notation, $\Gamma^{n}$ is used to denote any element of $\mathbb{P}_n$.
    \label{def_discounted_kernel_set}
\end{definition}

\begin{proposition}[Lemma 4, \citet{scherrer2015approximate}]
    After $k$ iterations of approximate modified policy iteration scheme with standard bellman operator $\mathcal{T}^{\pi}_{\mathrm{st}}$, the losses $l_k \coloneqq v_* - v_{\pi^k}$ satisfy \begin{equation}
        l_k \le 2 \sum_{i = 1}^{k - 1} \sum_{j = i}^{\infty} \Gamma^{j}|\epsilon_{k - i}| + \sum_{i = 0}^{k - 1}\sum_{j = i}^{\infty} \Gamma^{j}|\epsilon'_{k - i}| + h(k)
    \end{equation} where $h(k) \coloneqq 2 \sum_{j = k}^{\infty} \Gamma^{j}|d_0|$ or $h(k) \coloneqq 2 \sum_{j = k}^{\infty} \Gamma^{j}|b_0|$. \label{tech_l1}
\end{proposition}

\subsubsection{Proof pipelne of Theorem \ref{br_L_inf_error_bound}}

We first prove a technical result for our losses $l_k \coloneqq v_{\pi_t} - v_{\pi^k}$ that is similar to Proposition \ref{tech_l1}.

\begin{lemma} \label{lemma:inter-theorem}
    Under assumption \ref{assum_policy_realizability}, after $k$ iterations of scheme (\ref{ampi}), our losses satisfy \begin{equation}
        l_k \le 2 \sum_{i = 1}^{k - 1} \sum_{j = i}^{\infty} \Gamma^{j}|\epsilon_{k - i}| + \sum_{i = 0}^{k - 1}\sum_{j = i}^{\infty} \Gamma^{j}|\epsilon'_{k - i}| + h(k),
    \end{equation} where $h(k) \coloneqq 2 \sum_{j = k}^{\infty} \Gamma^{j}|l_0|$ or $h(k) \coloneqq 2 \sum_{j = k}^{\infty} \Gamma^{j}|b_0|$, with $b_0 = \mathcal{T}^{\pi^{1}}v^{0} - v^{0}$ that is related to the choice of the starting point. 
\end{lemma}

\begin{proof}
    Similar to the proofs in \citet{scherrer2015approximate} for Proposition \ref{tech_l1}, we first derive three relations that are identical to those in Proposition \ref{tech_inequality}, then apply these relations to bound our losses $l_k \coloneqq v_{\pi_t} - v_{\pi^k}$.
    
    To start with, we define 
    \begin{align}
        b_k \coloneqq v^k - \mathcal{T}^{\pi^{k + 1}} v^{k}, \quad s_k \coloneqq (\mathcal{T}^{\pi^{k}})^m v^{k - 1} - v_{\pi^k}, \quad d_k \coloneqq v_{\pi_t} - (\mathcal{T}^{\pi^{k}})^m v^{k - 1} .
    \end{align}

    \paragraph{Bounding $b_k$} With the definitions of $\epsilon_k$ and $\epsilon'_k$, and the property that \begin{equation}
        (\mathcal{T}^{\pi})^m v - (\mathcal{T}^{\pi})^m v' = (\gamma P^o_{\pi})^m(v - v'), \label{temp4}
    \end{equation}we have, \begin{align}
        b_k =&~ v^k - \mathcal{T}^{\pi^k} v^k + \mathcal{T}^{\pi^k} v^k - \mathcal{T}^{\pi^{k + 1}} v^k\\
        \overset{\mathrm{(i)}}{\le}&~ v^k - \mathcal{T}^{\pi^k} v^k + \epsilon'_{k + 1}\\
        =&~ v^k - \epsilon_k - \mathcal{T}^{\pi^k} v^k + P^o_{\pi^k}\epsilon_k + \epsilon_k - P^o_{\pi^k}\epsilon_k + \epsilon'_{k + 1}\\
        \overset{\mathrm{(ii)}}{=}&~ v^k - \epsilon_k - \mathcal{T}^{\pi^k}(v^k - \epsilon_k) + (I - \gamma P^o_{\pi^k})\epsilon_k + \epsilon'_{k + 1} \\
        =&~ v^k - \epsilon_k - \mathcal{T}^{\pi^k}(v^k - \epsilon_k) + x_k ,
    \end{align} where $x_k \coloneqq (I - \gamma P^o_{\pi^k})\epsilon_k + \epsilon'_{k + 1}$, inequality (i) comes from the definition of $\epsilon'_{k + 1}$ in (\ref{ampi}), equality (ii) comes from property (\ref{temp4}).
    \begin{align}
        b_k \le&~ v^k - \epsilon_k - \mathcal{T}^{\pi^k}(v^k - \epsilon_k) + x_k\\
        \overset{\mathrm{(i)}}{=}&~(\mathcal{T}^{\pi^k})^m v^{k - 1} - \mathcal{T}^{\pi^k} (\mathcal{T}^{\pi^k})^m v^{k - 1} + x_k\\
        =&~ (\mathcal{T}^{\pi^k})^m v^{k - 1} -  (\mathcal{T}^{\pi^k})^m (\mathcal{T}^{\pi^k} v^{k - 1}) + x_k\\
        \overset{\mathrm{(ii)}}{\le}&~ (\gamma P^o_{\pi^k})^m (v^{k - 1} - \mathcal{T}^{\pi^k} v^{k - 1}) + x_k \\
        =&~ (\gamma P^o_{\pi^k})^m b_{k - 1} + x_k,\label{eq_bk}
    \end{align} where equality (i) comes from the definition of $\epsilon$ in (\ref{ampi}), and equality (ii) comes from property (\ref{temp4}).
    
    \paragraph{Bounding $s_k$} With the property that $\forall v, v_{\pi^k} = (\mathcal{T}^{\pi^k})^{\infty} v$, we have \begin{align}
        s_k =&~ (\mathcal{T}^{\pi^{k}})^m v^{k - 1} - (\mathcal{T}^{\pi^k})^{\infty} v^{k - 1}\\
        =&~ (\gamma P^o_{\pi^k})^m(I - \gamma P^o_{\pi^k})^{-1} b_{k - 1}, \label{eq_sk}
    \end{align} where the last equality comes from property (\ref{temp4}).

    \paragraph{Bounding $d_k$} Define $y_k \coloneqq -\gamma P^o_{\pi_t}\epsilon_k + \epsilon'_{k + 1}$, then \begin{align}
        d_{k + 1} = &~ v_{\pi_t} - (\mathcal{T}^{\pi^{k + 1}})^m v^{k}\\
        \overset{\mathrm{(i)}}{=}&~ \mathcal{T}^{\pi_t} v_{\pi_t} - \mathcal{T}^{\pi_t} v^k + \mathcal{T}^{\pi_t} v^k - \mathcal{T}^{\pi^{k + 1}} v^k + \mathcal{T}^{\pi^{k + 1}} v^k - (\mathcal{T}^{\pi^{k + 1}})^m v^{k} \\
        \overset{\mathrm{(i)}}{\le}&~ \gamma P^o_{\pi_t}(v_{\pi_t} - v_k) + \epsilon'_{k + 1}  + \mathcal{T}^{\pi^{k + 1}} v^k - (\mathcal{T}^{\pi^{k + 1}})^m v^{k}\\
        =&~ \gamma P^o_{\pi_t}(v_{\pi_t} - v_k) + \gamma P^o_{\pi_t} \epsilon_k - \gamma P^o_{\pi_t} \epsilon_k + \epsilon'_{k + 1}  + \mathcal{T}^{\pi^{k + 1}} v^k - (\mathcal{T}^{\pi^{k + 1}})^m v^{k} \\
        =&~ \gamma P^o_{\pi_t} d_k + y_k + \mathcal{T}^{\pi^{k + 1}} v^k - (\mathcal{T}^{\pi^{k + 1}})^m v^{k}\\
        \overset{\mathrm{(iii)}}{=}&~ \gamma P^o_{\pi_t} d_k + y_k + \sum_{j = 1}^{m - 1}(\gamma P^o_{\pi^{k + 1}})^j b_k, \label{eq_dk}
    \end{align} where equality (i) comes from $v_{\pi_t} = \mathcal{T}^{\pi_t} v_{\pi_t}$, inequality (ii) comes from $\mathcal{T}^{\pi_t} v_{\pi_t} - \mathcal{T}^{\pi_t} v^k = \gamma P^o_{\pi_t}(v_{\pi_t} - v_k)$ and the definition of $\epsilon'_{k + 1}$, and equality (iii) comes from iteratively applying property of $(\mathcal{T}^{\pi})^j v - (\mathcal{T}^{\pi})^j v' = (\gamma P^o_{\pi})^j(v - v')$.

    After bounding all three elements, using the notation introduced in Definition \ref{def_discounted_kernel_set}, we may rewrite (\ref{eq_sk}) as \begin{equation}
        s_k = \Gamma^{m} \sum_{j = 0}^{\infty} \Gamma^j b_{k - 1},
    \end{equation} 
    And by induction from Eq.(\ref{eq_bk}) and (\ref{eq_dk}), we obtain \begin{align}
        b_k &\le \sum_{i = 1}^k \Gamma^{m(k - i)} x_i + \Gamma^{mk} b_0,\\
        d_k \le \sum_{j = 1}^{k - 1} &\Gamma^{k - 1 - j} (y_j + \sum_{l = 1}^{m - 1} \Gamma^{l} b_j) + \Gamma^{k}d_0.
    \end{align}
    
    Combine with the fact of $l_k = s_k + d_k$, now we recover exactl the same expression of $s_k$, $b_k$, and $d_k$ as in the proof of proposition \ref{tech_l1}. All the rest are therefore standard as in theirs except for the relation between $b_0$ and $d_0$, \begin{align}
        b_0 =&~ v^0 - \mathcal{T}^{\pi^1} v^0 \\
        =&~ v^0 - v_{\pi_t} + \mathcal{T}^{\pi_t}v_{\pi_t} - \mathcal{\pi_t} v^0 + \mathcal{\pi_t} v^0 - \mathcal{T}^{\pi^1} v^0 \\
        \le&~ (I - \gamma P^o_{\pi_t})(-d_0) + \epsilon'_{1},
    \end{align} where the fact of $\epsilon_0 = 0$ is used and we again recover the same relation. 
\end{proof}

Finally we finish the proof of theorem \ref{br_L_inf_error_bound} by noticing that $\Gamma^{n} \le \gamma^{n}$ in Definition \ref{def_discounted_kernel_set}. Therefore, we have \begin{align}
    l_k \le&~ 2 \sum_{i = 1}^{k - 1} \sum_{j = i}^{\infty} \Gamma^{j}|\epsilon_{k - i}| + \sum_{i = 0}^{k - 1}\sum_{j = i}^{\infty} \Gamma^{j}|\epsilon'_{k - i}| + h(k)\\
    \le&~ \frac{2\bar{\epsilon}(\gamma-\gamma^k) + \bar{\epsilon'} (1-\gamma^k)}{(1-\gamma)^2} + O(\gamma^k) \\
    \xrightarrow{k \to \infty}&~ \frac{2\bar{\epsilon}\gamma+\bar{\epsilon'}}{(1-\gamma)^2}.
\end{align}
\qed

\section{Derivation of Section \ref{sec:contamination} : DRC-RL with R-Contamination Uncertainty Sets}
\label{sec:derivation_contamination}

There exist a variety class of operators of our proposed operator $\mathcal{T}^{\pi}$ in (\ref{general_operator}), that keep properties as in Proposition \ref{basic_properties_T}. We present one of them and upon which derive the exact consistency operator (\ref{contamination_operator}) used for R-contamination uncertainty set in Section \ref{secy:contamination}. 

Formally, $\forall \lambda_t \in \mathbb{R}_{+}, s \in \mathcal{S}, v \in \RSt$, we define a class of consistency operators $[\mathcal{T}_{1}^{\pi}v](s) \coloneqq [\mathcal{T}^{\pi}v](s) - \gamma h\langle P^o_{s, \pi(s)}, v - V^{\pi}_r + \lambda_t^{\top} V^{\pi}_g\rangle$ controlled by coefficient $h \in [0, 1)$. We use $P^o_{\pi}$ to denote vector $[P^o_{s, \pi(s)}]^{\top}_{s \in \mathcal{S}}$.

\begin{proposition}
    The consistency operator $\mathcal{T}_{1}^{\pi}$ satisfies the following :\begin{enumerate}
        \item[(1)] Monotonicity: let $v_1, v_2 \in \mathbb{R}^{\mathcal{S}}$ such that $v_1 \ge v_2$, then $\mathcal{T}_1^{\pi} v_1 \ge \mathcal{T}_1^{\pi} v_2$. 
        \item[(2)] Transition Invariance: for any $c \in \mathbb{R}$, we have $\mathcal{T}_1^{\pi} (v + c\textbf{1}) = \mathcal{T}_1^{\pi} v + \gamma(1 - h) c\textbf{1}$.
        \item[(3)] Contraction: The operator $\mathcal{T}_1^{\pi}$ is a $\gamma(1 - h)$-contraction. Further, $V^{\pi}_{r} - \lambda_t^{\top}V^{\pi}_{g}$ is the unique stationary point. 
    \end{enumerate} 
    \label{prop_contamination_operator}
\end{proposition}

\begin{proof}
    \textbf{Monotonicity}. By the definition of $\mathcal{T}_1^{\pi}$ , we have $\forall v_1, v_2 \in \RSt$ such that $v_1 \ge v_2$, \begin{equation}
        \mathcal{T}_1^{\pi} v_1 - \mathcal{T}_1^{\pi} v_2 = \gamma(1 - h) \langle P^o_{\pi}, v_1 - v_2 \rangle \ge 0.
    \end{equation}
    \textbf{Transition Invariance}. By the definition of $\mathcal{T}_1^{\pi}$, we have $\forall v_1 \in \RSt, c \in \mathbb{R}$ \begin{equation}
        \mathcal{T}_1^{\pi} (v_1 + c\mathbf{1}) = \mathcal{T}_1^{\pi} v_1 + \gamma(1 - h) \langle P^o_{\pi}, c\mathbf{1} \rangle = \mathcal{T}_1^{\pi} v_1 + \gamma(1 - h) c\mathbf{1}.
    \end{equation}
    \textbf{Contraction}. We first show that $\mathcal{T}_{1}^{\pi}$ is a $\gamma (1 - h)$ contraction by its definition, that we have $\forall v_1, v_2 \in \RSt, s \in \mathcal{S}$, \begin{equation}
        |[\mathcal{T}_1^{\pi} v_1](s) - [\mathcal{T}_1^{\pi} v_2](s)| = |\gamma(1 - h) \langle P^0_{s, \pi(s)}, v_1 - v_2 \rangle|
        \le \gamma(1 - h) \Vert v_1 - v_2\Vert_{\infty}.
    \end{equation}
    Finally, we apply $\mathcal{T}_{1}^{\pi}$ on $V^{\pi}_{r} - \lambda_t^{\top}V^{\pi}_{g}$, \begin{align}
        [\mathcal{T}_{1}^{\pi} (V^{\pi}_{r} - \lambda_t^{\top}V^{\pi}_{g})](s) =&~ [\mathcal{T}^{\pi} (V^{\pi}_{r} - \lambda_t^{\top}V^{\pi}_{g})](s) - \gamma h\langle P^o_{s, \pi(s)}, V^{\pi}_{r} - \lambda_t^{\top}V^{\pi}_{g} - V^{\pi}_r + \lambda_t^{\top} V^{\pi}_g\rangle\\
        =&~ [\mathcal{T}^{\pi} (V^{\pi}_{r} - \lambda_t^{\top}V^{\pi}_{g})](s) \\
        =&~ V^{\pi}_{r}(s) - \lambda_t^{\top}V^{\pi}_{g}(s),
    \end{align} where the first equality comes from the definition of $\mathcal{T}_{1}^{\pi}$, the last equality comes from the contraction property of $\mathcal{T}^{\pi}$ in Proposition \ref{basic_properties_T}.
\end{proof}

By setting $h = \beta$ and specifying the uncertainty set as R-contamination uncertainty set $\mathcal{P}_{s, a} = \{(1 - \beta)P^o_{s, a} + \beta q \,|\, q \in \triangle(\mathcal{S})\}$ with a scalar robust level $\beta \in \mathbb{R}$, we have for $\mathcal{T}_1^{\pi}$ defined in Proposition \ref{prop_contamination_operator} and $s \in \mathcal{S}, v \in \RSt$, \begin{align}
    [\mathcal{T}_{1}^{\pi}v](s) =&~ [\mathcal{T}^{\pi}v](s) - \gamma \beta\langle P^o_{s, \pi(s)}, v - V^{\pi}_r + \lambda_t^{\top} V^{\pi}_g \rangle\\
    \overset{\mathrm{(i)}}{=}&~ (r - \lambda_t^{\top}g)(s, \pi(s)) + \gamma \langle P^o_{s, \pi(s)}, v \rangle +  \gamma \min_{P \in \mathcal{P}} \langle P_{s, \pi(s)} - P^o_{s, \pi(s)}, V^{\pi}_{r} \rangle  \\
    &~- \gamma \lambda_t^{\top} \min_{P \in \mathcal{P}} \langle P_{s, \pi(s)} - P^o_{s, \pi(s)}, V^{\pi}_{g} \rangle - \gamma \beta\langle P^o_{s, \pi(s)}, v - V^{\pi}_r + \lambda_t^{\top} V^{\pi}_g\rangle \nonumber\\
    \overset{\mathrm{(ii)}}{=}&~ (r - \lambda_t^{\top}g)(s, \pi(s)) + \gamma \langle P^o_{s, \pi(s)}, v \rangle +  \gamma \min_{q \in \Delta(\cS)} \langle \beta(q - P^o_{s, \pi(s)}), V^{\pi}_{r} \rangle   \\
    &~- \gamma \lambda_t^{\top} \min_{q \in \Delta(\cS)} \langle \beta(q - P^o_{s, \pi(s)}), V^{\pi}_{g} \rangle - \gamma \beta\langle P^o_{s, \pi(s)}, v - V^{\pi}_r + \lambda_t^{\top} V^{\pi}_g\rangle \nonumber\\
    =&~ (r - \lambda_t^{\top}g)(s, \pi(s)) + \gamma(1 - \beta) \langle P^o_{s, \pi(s)}, v \rangle + \gamma \beta (\min_{q \in \Delta(\cS)} \langle q, V^{\pi}_{r} \rangle - \lambda_t^{\top} \min_{q \in \Delta(\cS)} \langle q, V^{\pi}_{g} \rangle ) \\
    =&~ (r - \lambda_t^{\top}g)(s, \pi(s)) + \gamma(1 - \beta) \langle P^o_{s, \pi(s)}, v \rangle + \gamma \beta (\min_{s'} V^{\pi}_r(s') - \lambda_t^{\top} \min_{s'}V^{\pi}_{g}(s')). 
\end{align} where equality (i) comes from the definition of consistency operator $\mathcal{T}^{\pi}$ in (\ref{general_operator}), and equality (ii) comes from the definition of R-contamination uncertainty set.

Now we have obtained the same consistency operator as in (\ref{contamination_operator}) that is derived from proposed consistency operator $\mathcal{T}^{\pi}$ in (\ref{general_operator}), so we use the same notation as $\mathcal{T}^{\pi}$ in the main text for brevity sake (and one can certainly proof everything in Section \ref{sec:framework} as they mostly depend on these three properties, but it is no need to do that as following).

In the following content of Section \ref{sec:contamination}, we will further simplify to obtain a standard consistency operator, thus the convergence result is standard from \citet{scherrer2015approximate}, such as in Proposition \ref{tech_l1}.

\section{Proof of Section~\ref{sec:inherent}: On the Intractability of DRC-RL}
\label{appendix:proof_5}

\subsection{Proof  of Lemma \ref{eq_linear_greedy}}

For brevity, we denote policy as $\pi \in \Delta(\mathcal{A})^{\mathcal{S}} \subseteq \RAS$ which includes both the class of stochastic and deterministic policies. We therefore use $\pi[s, a]$ to denote the probability of action $a$ given state $s$, and use $\pi(s)$ to denote the whole probability simplex given state $s$.

When an operator is linear, by the definition of linearity, there exists a function $f : \mathbb{R}^{|\mathcal{S}|} \to \mathbb{R}^{|\mathcal{S}||\mathcal{A}|}$, such that for every value function $v \in \mathbb{R}^{|\mathcal{S}|}$ and state $s \in \mathcal{S}$, \begin{align}
    \max_{\pi \in \Pi} [\mathcal{T}^{\pi}v](s) =&~ \max_{\pi \in \Pi} \langle \pi, f(v)\rangle_s\\
    =&~ \max_{\pi \in \Pi} \langle \pi[s, \cdot], f(v)[s, \cdot]\rangle\\
    =&~ \max_{\pi(s) \in \Delta(\mathcal{A})} \langle \pi(s), f(v)[s, \cdot] \rangle\\
    =&~ \max_{a} f(v)[s, a]
\end{align}
The greedy policy can be set as $g(v, a)_s = f(v)[s, a]$, which completes the proof for the first part. 

When an operator $\mathcal{T}^{\pi}$ enables the greedy policy, by its definition, there exists a function $g : \mathbb{R}^{|\mathcal{S}|} \times \mathcal{A} \to \mathbb{R}^{|\mathcal{S}|}$ such that $\forall v \in \mathbb{R}^{|\mathcal{S}|}, s \in \mathcal{S}$, $\max_{\pi \in \Pi} [\mathcal{T}^{\pi}v](s) = \max_{a \in \mathcal{A}}g(v, a)_{s}$. Therefore, when defining the linear operator $\mathcal{T}^{\pi}_{linear} \coloneqq \langle \pi(s), g(v, \cdot) \rangle$, we complete the proof by showing:  \begin{align}
    \max_{\pi \in \Pi} [\mathcal{T}^{\pi}v](s) =&~ \max_{a \in \mathcal{A}}g(v, a)_{s} \\
    =&~ \max_{\pi(s) \in \Delta(\mathcal{A})} \langle \pi(s), g(v, \cdot) \rangle\\
    =&~ \max_{\pi \in \Pi} \langle \pi(s), g(v, \cdot) \rangle\\
    =&~ \max_{\pi \in \Pi} [\mathcal{T}^{\pi}_{linear}v](s).
\end{align} 
\qed

\subsection{Proof for Theorem \ref{impossible_thm}}

\label{proof:inherent}


Throughout this proof,  for convenience, we denote $v^{\pi} \coloneqq V^{\pi}_r - \lambda_t^T V^{\pi}_g$ for any policy $\pi\in\Pi$ and any $\lambda_t$. 

To start with, we suppose that there exists a consistency operator denoted as $\mathcal{T}'$ that satisfies both conditions in Theorem \ref{impossible_thm}
for all policy $\pi$, namely, $\forall s \in \St, \pi_1, \pi_2 \in \Pi$, we have: \begin{align}
    |[\mathcal{T}'^{\pi_1} v^{\pi_1}](s) - [\mathcal{T}'^{\pi_2} v^{\pi_2}](s)| =&~ |[\mathcal{T}'^{\pi_1} v^{\pi_1}](s) - [\mathcal{T}'^{\pi_1} v^{\pi_2}](s) + [\mathcal{T}'^{\pi_1} v^{\pi_2}](s) - [\mathcal{T}'^{\pi_2} v^{\pi_2}](s)|\\
    \overset{\mathrm{(i)}}{\le}&~ |[\mathcal{T}'^{\pi_1} v^{\pi_1}](s) - [\mathcal{T}'^{\pi_1} v^{\pi_2}](s)| + |[\mathcal{T}'^{\pi_1} v^{\pi_2}](s) - [\mathcal{T}'^{\pi_2} v^{\pi_2}](s)|\\
    \overset{\mathrm{(ii)}}{\le}&~ \gamma \left\Vert v^{\pi_1} - v^{\pi_2} \right\Vert_{\infty} + |\left\langle \pi_1(s) - \pi_2(s), f(v^{\pi_2})[s, \cdot]\right\rangle|, \label{eq:key-inequality-drc}
\end{align} where inequality (i) uses the triangle inequality, inequality (ii) comes from the contraction property in Proposition \ref{impossible_thm} and the linearity property defined in definition \ref{linearity_definition}.

Moreover, as $v^{\pi}$ is the stationary point of consistency operator $\mathcal{T}'^{\pi}$, we arrive at an inequality that is independent of any operator,\begin{align}
    |v^{\pi_1}(s) - v^{\pi_2}(s)| 
    =&~ |[\mathcal{T}'^{\pi_1} v^{\pi_1}](s) - [\mathcal{T}'^{\pi_2} v^{\pi_2}](s)| \\
    \le&~\gamma \left\Vert v^{\pi_1} - v^{\pi_2} \right\Vert_{\infty} + |\left\langle \pi_1(s) - \pi_2(s), f(v^{\pi_2})[s, \cdot]\right\rangle|,\label{contradiction_ineq}
\end{align}
where the last inequality holds by applying \eqref{eq:key-inequality-drc}.

For the rest of the proof, we will construct an example that contradicts with (\ref{contradiction_ineq}), which shows that there is no such operator $\cT'$ that satisfies two conditions introduced in Theorem \ref{impossible_thm}.

 \begin{figure}[!h]
     \centering
     \includegraphics[width = .6\textwidth]{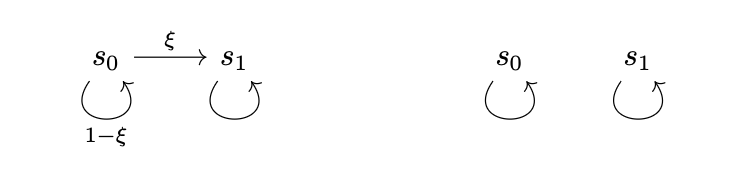}
     \caption{A two states, two actions Markov decision process used in example \ref{example} : the left and the right figures present the transition probabilities for actions $a_0$ and $a_1$.}
     \label{fig:enter-label}
 \end{figure}
 \begin{example}
    \label{example}
    Consider Markov decision process as in figure \ref{fig:enter-label}, where state space $\St = \{s_0, s_1\}$, action space $\Ac = \{a_0, a_1\}$, discount factor $\gamma = 0.95$. The transition kernel for state $s_0$ and action $a_0$ is parameterized by $\xi$ that indicates the probability to state $s_1$, i.e. $\mathcal{P}_{s_0, a_0} = \{[1-\xi, \xi] | \xi \in [0.9, 1]\}$. For other state action pairs, the only consequence is staye at the current state, i.e. $\mathcal{P}_{s_1, a_0} = \mathcal{P}_{s_1, a_1} = \{[0, 1]\}, \mathcal{P}_{s_0, a_1} = \{[1, 0]\}$. 
    
    We assume there is one additional constraint, and design the rewards $r$ and $g$ for the objective and the constraint respectively,
    \begin{equation}
        r_{s, a} = \begin{cases}
            1, & \text{ if } s = s_0, a = a_0\\
            0, & \text{ if }s = s_0, a = a_1\\
            1, & \text{ if }s = s_1, a = a_0\\
            2, & \text{ if } s = s_1, a = a_1
        \end{cases} 
        ~~~~
        g_{s, a} = \begin{cases}
            1, & \text{ if }s = s_0, a = a_0\\
            0, & \text{ if }s = s_0, a = a_1\\
            0, & \text{ if }s = s_1, a = a_0\\
            1, & \text{ if } s = s_1, a = a_1
        \end{cases}
    \end{equation}
    The derived inequality (\ref{contradiction_ineq}) should hold for any choice of policies $\pi_1$ and $\pi_2$, we thus choose \begin{equation}
        \pi_1(s) = \begin{cases}
            a_0, & \text{ if }s = s_0\\
            a_0, & \text{ if } s = s_1
        \end{cases}
        ~~~~
        \pi_2(s) = \begin{cases}
            a_0, & \text{ if } s = s_0\\
            a_1, & \text{ if } s = s_1
        \end{cases}
    \end{equation}
    
    The derived inequality (\ref{contradiction_ineq}) then becomes, \begin{align}
        |v^{\pi_1}(s_0) - v^{\pi_2}(s_0)|  \le&~ \gamma \left\Vert v^{\pi_1} - v^{\pi_2} \right\Vert_{\infty} + |\left\langle \pi_1(s_0) - \pi_2(s_0), f(v^{\pi_2})[s_0, \cdot]\right\rangle| \\
        =&~ \gamma \left\Vert v^{\pi_1} - v^{\pi_2} \right\Vert_{\infty} \label{contradiction_simplify}
    \end{align} where the second equation comes from the fact that $\pi_1(s_0) - \pi_2(s_0) = 0$.

    One can then calculate the robust value function $V^{\pi_1}_r$, 
    \begin{align}
        V^{\pi_1}_r(s_1) =&~ \frac{r_{s_1, \pi_1(s_1)}}{1 - \gamma} = 20 \\
        V^{\pi_1}_r(s_0) =&~ [\mathcal{T}^{\pi_1}_{r, \mathrm{rob}} V^{\pi_1}_r](s_0)\\
        =&~ r_{s_0, \pi_1(s_0)} + \gamma \min_{P \in \mathcal{P}_{s_0, \pi_1(s_0)}} \langle P, V^{\pi_1}_r \rangle \\
        =&~ 1 + \gamma \min_{\xi \in [0.9, 1]} \{(1 - \xi) V^{\pi_1}_r(s_0) + \xi V^{\pi_1}_r(s_1)\} \\
        =&~ 1 + \gamma \min_{\xi \in [0.9, 1]} \{(1 - \xi) V^{\pi_1}_r(s_0) + 20\xi\} \label{temp1}\\
        =&~ 1 + \gamma \begin{cases}
            20, &V^{\pi_1}_r(s_0) \ge 20 \\
            0.1V^{\pi_1}_r(s_0) + 18, &V^{\pi_1}_r(s_0) < 20
        \end{cases} \label{temp2}
    \end{align} where the first two equations in calculating $V^{\pi_1}_r(s_0)$ comes from the fact that $V^{\pi_1}_r$ is the stationary point of consistency operator $\mathcal{T}^{\pi_1}_{r, \mathrm{rob}}$ in (\ref{dr_operator}), and the last equation illustrates the solution of the inner minimization problem in (\ref{temp1}).

    One can easily calculate both cases in (\ref{temp2}) and find that $V^{\pi_1}_r(s_0) = 20$. Similarly, we can calculate other robust value functions, \begin{align}
        V^{\pi_1}_g(s_1) =&~ \frac{g_{s_1, \pi_1(s_1)}}{1 - \gamma} = 0\\
        V^{\pi_1}_g(s_0) =&~ g_{s_0, \pi_1(s_0)} + \gamma \min_{\xi \in [0.9, 1]} \{(1 - \xi) V^{\pi_1}_g(s_0) + \xi V^{\pi_1}_g(s_1)\} = 1\\
        V^{\pi_2}_r(s_1) =&~ \frac{r_{s_1, \pi_2(s_1)}}{1 - \gamma} = 40\\
        V^{\pi_2}_r(s_0) =&~ r_{s_0, \pi_2(s_0)} + \gamma \min_{\xi \in [0.9, 1]} \{(1 - \xi) V^{\pi_2}_r(s_0) + \xi V^{\pi_2}_r(s_1)\} \\
        =&~ 0 + \gamma \begin{cases}
            40, &V^{\pi_1}_r(s_0) \ge 40 \\
            0.1V^{\pi_1}_r(s_0) + 36, &V^{\pi_1}_r(s_0) < 40
        \end{cases} \\
        =&~ \frac{6840}{181} \approx 37.79 \\
        V^{\pi_2}_g(s_1) =&~ \frac{g_{s_1, \pi_2(s_1)}}{1 - \gamma} = 20\\
        V^{\pi_2}_g(s_0) =&~ g_{s_0, \pi_2(s_0)} + \gamma \min_{\xi \in [0.9, 1]} \{(1 - \xi) V^{\pi_2}_g(s_0) + \xi V^{\pi_2}_g(s_1)\} = 20
    \end{align}
    Therefore, we have 
    \begin{align}
        v^{\pi_1} = V^{\pi_1}_r - \lambda_t V^{\pi_1}_g = \left[ 
            \begin{array}{c}
                20 \\
                20 
            \end{array}
        \right] - \lambda_t \left[ 
            \begin{array}{c}
                1 \\
                0 
            \end{array}
        \right] = \left[ 
            \begin{array}{c}
                20 - \lambda_t \\
                20 
            \end{array}
        \right]\\
        v^{\pi_2} = V^{\pi_2}_r - \lambda_t V^{\pi_2}_g = \left[ 
            \begin{array}{c}
                \frac{6840}{181} \\
                40 
            \end{array}
        \right] - \lambda_t \left[ 
            \begin{array}{c}
                20 \\
                20 
            \end{array}
        \right] = \left[ 
            \begin{array}{c}
                \frac{6840}{181} - 20 \lambda_t \\
                40 - 20 \lambda_t
            \end{array}
        \right]
    \end{align}
    Finally, for every $\lambda_t \in [0.969, 2.209]$, we have \begin{equation}
        |v^{\pi_1}(s_0) - v^{\pi_2}(s_0)| = \left\Vert v^{\pi_1} - v^{\pi_2} \right\Vert_{\infty} > \gamma \left\Vert v^{\pi_1} - v^{\pi_2} \right\Vert_{\infty}
    \end{equation} which gives a clear contradiction to derived inequality (\ref{contradiction_simplify}).
    
\end{example}
\qed

\begin{remark}
    The simplified derived inequality (\ref{contradiction_simplify}) is true to the non-robust or non-constrained counterparts of DRC-RL.
\end{remark}


\section{Additional Discussions}

\subsection{Discussions about Assumption \ref{assum_policy_realizability}}
\label{appendix:1}

We start from the definitions of our consistency operator $\mathcal{T}^{\pi}$ in (\ref{general_operator}) and the corresponding optimality operator $\mathcal{T}^{*}$, \begin{align}
     [\mathcal{T}^{\pi} v](s) =&~ (r - \lambda_t^{\top}g)(s, \pi(s)) + \gamma \langle P^o_{s, \pi(s)}, v \rangle +  \gamma \min_{P \in \mathcal{P}} \langle P_{s, \pi(s)} - P^o_{s, \pi(s)}, V^{\pi}_{r} \rangle  \label{general_operator} \nonumber \\
    &~- \gamma \lambda_t^{\top} \min_{P \in \mathcal{P}} \langle P_{s, \pi(s)} - P^o_{s, \pi(s)}, V^{\pi}_{g} \rangle \nonumber\\
    [\mathcal{T}^{*} v](s) =&~ \max_{\pi \in \Pi} [\mathcal{T}^{\pi} v](s) \nonumber
\end{align}

Given that consistency operator $\mathcal{T}^{\pi}$ does not have linearity defined in definition \ref{linearity_definition}, the optimality operator $\mathcal{T}^{*}$ cannot take greedy policy as proved in lemma \ref{eq_linear_greedy}. Further, for a fixed value function $v$ and different states $s$ and $s'$, there might not exist a single policy that can be returned by the optimality operator $\mathcal{T}^{*}$ acting on both states, i.e. maximize $\mathcal{T}^{\pi} v(s)$ and $\mathcal{T}^{\pi} v(s')$ at the same time. Therefore, assumption \ref{assum_policy_realizability} for the policy improvement step actually consists of two statements : (1) There exists a policy that can be $\epsilon'$-approximate to the optimality operator. (2) There exists a solver that efficiently finds this policy.

\textbf{In this discussion, we will introduce a modified operator that guarantees the existence of a $\epsilon'$-approximate policy to the optimality operator. Then further build the related AMPI scheme to solve distributionally robust constrained RL.}. To start with, we first denote $v(\mu) \coloneqq \langle v, \mu \rangle$ for any vector $v \in \RSt$. We then define the $\mu$-consistency operator $\mathcal{T}^{\pi}_{\mu} : \mathbb{R} \to \mathbb{R}$, \begin{equation}
    \mathcal{T}^{\pi}_{\mu} v(\mu) = [\mathcal{T}^{\pi} v](\mu)
\end{equation} and its corresponding $\mu$-optimality operator is, \begin{equation}
    \mathcal{T}^{*}_{\mu} v(\mu) = \max_{\pi \in \Pi} \mathcal{T}^{\pi}_{\mu} v(\mu) = \max_{\pi \in \Pi} \left\{ [\mathcal{T}^{\pi} v](\mu) \right\}
\end{equation}

As $\mu$-optimality operator now gives a scalar, for any value function $v \in \RSt$, there always exist a policy $\pi_v$ such that $\mathcal{T}^{*}_{\mu} v(\mu) = \mathcal{T}^{\pi_v}_{\mu} v(\mu)$. And the existence of $\epsilon'$-approximate policy in the policy improvement step is automatically proved, and we further have the existence of optimal policy $\pi^* = \pi_t$.

\begin{proposition}
    The $\mu$-consistency operator $\mathcal{T}^{\pi}_{\mu}$ and the optimality operator $\mathcal{T}^{*}$ have the following properties, \begin{enumerate}
        \item Monotonicity: let $v_1, v_2 \in \mathbb{R}^{|\mathcal{S}|}$ such that $v_1(\mu) \ge v_2(\mu)$, then $\mathcal{T}^{\pi}_{\mu} v_1(\mu) \ge \mathcal{T}^{\pi}_{\mu} v_2(\mu)$ and $\mathcal{T}^{*}_{\mu} v_1(\mu) \ge \mathcal{T}^{*}_{\mu} v_2(\mu)$ . 
        \item Transition Invariance: for any $c \in \mathbb{R}$, we have $\mathcal{T}^{\pi}_{\mu} (v + c\textbf{1})(\mu) = \mathcal{T}^{\pi}_{\mu} v(\mu) + \gamma c\textbf{1}(\mu)$ and $\mathcal{T}^{*}_{\mu} (v + c\textbf{1})(\mu) = \mathcal{T}^{*}_{\mu} v(\mu) + \gamma c\textbf{1}(\mu)$
        \item Contraction: The operator $\mathcal{T}^{\pi}_{\mu}$ and $\mathcal{T}^{*}_{\mu}$ are $\gamma$-contractions, whose unique stationary points are $(V^{\pi}_{r} - \lambda_t^{\top}V^{\pi}_{g})(\mu)$ and $(V^{\pi_t}_{r} - \lambda_t^{\top}V^{\pi_t}_{g})(\mu)$ respectively. 
    \end{enumerate} \label{basic_properties_T_mu}
\end{proposition}

\begin{proof}
    By the definition of the $\mu$-consistency operator $\mathcal{T}^{\pi}_{\mu} v(\mu) = [\mathcal{T}^{\pi} v](\mu)$, it is straightforward that the monotonicity, transition invariance, and contraction properties hold for $\mu$-consistency operator. And that the unique stationary point of $\mu$-consistency operator is $(V^{\pi}_{r} - \lambda_t^{\top}V^{\pi}_{g})(\mu)$.

    We then consider $\mu$-optimality operator $\mathcal{T}^{*}_{\mu} v(\mu)$.

    \paragraph{Monotonicity} Let $\pi_1$ satisfies $\mathcal{T}^{*}_{\mu} v_1(\mu) = \mathcal{T}^{\pi_1}_{\mu} v_1(\mu)$ and similar to $\pi_2$, we have \begin{equation}
        \mathcal{T}^{*}_{\mu} v_1(\mu) - \mathcal{T}^{*}_{\mu} v_2(\mu) = \max_{\pi} \mathcal{T}^{\pi}_{\mu} v_1(\mu) - \mathcal{T}^{\pi_2}_{\mu} v_2(\mu) \ge \mathcal{T}^{\pi_2}_{\mu} v_1(\mu) - \mathcal{T}^{\pi_2}_{\mu} v_2(\mu) \ge 0
    \end{equation}

    \paragraph{Transition Invariance} \begin{equation}
        \mathcal{T}^{*}_{\mu} (v_1 + c \mathbf{1})(\mu) = \max_{\pi} \mathcal{T}^{\pi}_{\mu} (v_1 + c \mathbf{1})(\mu) = \max_{\pi} \mathcal{T}^{\pi}_{\mu} v_1(\mu) + \gamma c \mathbf{1}(\mu)
    \end{equation}

    \paragraph{Contraction} We first assume $\mathcal{T}^{*}_{\mu} v_1(\mu) > \mathcal{T}^{*}_{\mu} v_2(\mu)$ without loss of generality, \begin{equation}
        |\mathcal{T}^{*}_{\mu} v_1(\mu) - \mathcal{T}^{*}_{\mu} v_2(\mu)| = \mathcal{T}^{\pi_1}_{\mu} v_1(\mu) - \max_{\pi} \mathcal{T}^{\pi}_{\mu} v_2(\mu) \le \mathcal{T}^{\pi_1}_{\mu} v_1(\mu) - \mathcal{T}^{\pi_1}_{\mu} v_2(\mu) \le \gamma | v_1(\mu) - v_2(\mu) |
    \end{equation}
    Then as $(V^{\pi_t}_{r} - \lambda_t^{\top}V^{\pi_t}_{g})(\mu)$ is the unique stationary point of $\mathcal{T}^{\pi_t}_{\mu}$, we have $(V^{\pi_t}_{r} - \lambda_t^{\top}V^{\pi_t}_{g})(\mu) \le \mathcal{T}^{\pi^*}_{\mu} (V^{\pi_t}_{r} - \lambda_t^{\top}V^{\pi_t}_{g})(\mu) \le (V^{\pi^*}_{r} - \lambda_t^{\top}V^{\pi^*}_{g})(\mu)$ where $\pi^*$ is the stationary policy of $\mathcal{T}^{\pi^*}_{\mu}$.

    At the same time, by the definition of $\pi_t$, we have, \begin{equation}
        (V^{\pi_t}_{r} - \lambda_t^{\top}V^{\pi_t}_{g})(\mu) \ge (V^{\pi^*}_{r} - \lambda_t^{\top}V^{\pi^*}_{g})(\mu) 
    \end{equation} 
    Therefore, $(V^{\pi_t}_{r} - \lambda_t^{\top}V^{\pi_t}_{g})(\mu)$ is the unique stationary point of $\mathcal{T}^{\pi^*}_{\mu}$.
\end{proof}

Now consider the related AMPI scheme, with which the assumption \ref{assum_policy_realizability} only assumes the availability of approximate solvers, \begin{equation}
    \begin{cases}
        \pi^{k + 1} &=~~ \argmax^{\epsilon'_{k + 1, \mu}}_{\pi} \mathcal{T}_{\mu}^{\pi} v^k(\mu)\\
        v^{k + 1} &=~~ (\mathcal{T}^{\pi^{k + 1}})^m v^k + \epsilon_{k + 1}
    \end{cases} \label{ampi_mu}
\end{equation}

Unlike AMPI scheme (\ref{ampi}), the error for policy improvement step $\epsilon'_{\mu} \in \mathbb{R}$ is a scalar that guarantees $\max_{\pi} \mathcal{T}^{\pi}_{\mu} v^k(\mu) \le \mathcal{T}^{\pi^{k + 1}}_{\mu} v^k(\mu) + \epsilon'_{k + 1, \mu}$. We then define the non-negative scalar loss $l_{k, \mu} = v_{\pi_t}(\mu) - v_{\pi^k}(\mu)$ and finally obtain its absolute error bounds.

\begin{theorem}
    \label{mu_br_L_inf_error_bound}
    Under assumption \ref{assum_policy_realizability}, after $k$ iterations of scheme (\ref{ampi_mu}), the losses satisfy \begin{equation}
        |l_{k, \mu}| \le 2 \sum_{i = 1}^{k - 1} \sum_{j = i}^{\infty} \Gamma^{j}|\epsilon_{k - i}(\mu)| + \sum_{i = 0}^{k - 1}\sum_{j = i}^{\infty} \Gamma^{j}|\epsilon'_{k - i, \mu}| + h(k)
    \end{equation} where $h(k) \coloneqq 2 \sum_{j = k}^{\infty} \Gamma^{j}|l_{0, \mu}|$ or $h(k) \coloneqq 2 \sum_{j = k}^{\infty} \Gamma^{j}|b_{0, \mu}|$, with $b_{0, \mu} = (\mathcal{T}^{\pi^{1}}v^{0} - v^{0})(\mu)$ that is related to the choice of the starting point. 
\end{theorem} 

\begin{proof}
    Given the relation of $\mathcal{T}^{\pi}_{\mu} v(\mu) = [\mathcal{T}^{\pi} v](\mu)$ and the properties in proposition \ref{basic_properties_T_mu}, one can define $b_{k, \mu} = b_{k}(\mu)$ and similar for $s_{k, \mu}$ and $d_{k, \mu}$. Similar to proof for theorem \ref{br_L_inf_error_bound}, we have \begin{align}
        b_{k, \mu} &\le \sum_{i = 1}^k \Gamma^{m(k - i)} x_i(\mu) + \Gamma^{mk} b_{0, \mu}\\
        d_{k, \mu} &\le \sum_{j = 1}^{k - 1} \Gamma^{k - 1 - j} (y_j(\mu) + \sum_{l = 1}^{m - 1} \Gamma^{l} b_{j, \mu}) + \Gamma^{k}d_{0, \mu}\\
        s_{k, \mu} &= \Gamma^{m} \sum_{j = 0}^{\infty} \Gamma^j b_{k - 1, \mu}
    \end{align}and the relation between $b_{0, \mu}$ and $d_{0, \mu}$, \begin{equation}
        b_{0, \mu} \le (I - \gamma P^o_{\pi_t})(-d_{0, \mu}) + \epsilon'_{1, \mu}
    \end{equation}
    Then follow with the proof in proposition \ref{tech_l1}, we complete by noticing $l_{k, \mu}$ is non-negative given the definition of\BestResponse $\pi_t$.
\end{proof} 


\subsection{An Instantiation of the Solution with Contamination Set}
\label{appendix:2}

In this section, we give a specific instantiation for our solution to distributionally robust constrained RL (DRC-RL) with R-contamination uncertainty sets. In short, as our solution simplified to the non-robust counterpart of DRC-RL with asmaller discount factor, we use the same subroutines as in \citet{le2019batch} and get the theoretical results with generalization needs.

In specific,\BestResponse~is instatiated with fitted Q iteration (FQI) \citep{ernst2005tree},\OnlineAlgo~is chosen as exponentiated gradient (EG) \citep{kivinen1997exponentiated}. The EG algorithm requires bounded $\lambda$, we thus force $\Vert \lambda \Vert_1 \le B$. When evaluating Lagrangian, e.g. $L_{\max}$ and $L_{\min}$ in meta algorithm \ref{Meta_Algo}, fitted Q evaluation (FQE) \citep{le2019batch} is used. We represent the resulting instantiation in algorithm \ref{Algo_contamination}, which has its true performance that could be arbitrarily close to the optimal policy for DRC-RL.

\begin{algorithm}[t] 
\caption{DRC-RL with Contamination Set}
\label{Algo_contamination}
\textbf{Input} : Dataset $\mathcal{D} = \{x_i, a_i, x'_i, r_i, g_i\}_{i=1}^n$. Online algorithm parameters: $\ell_1$ norm bound $B$, learning rate $\eta$. Robust level $\beta$. Initial discounted horizon $\gamma_0.$
\begin{algorithmic}[1]
\State Initialize $\lambda_1 = (\frac{B}{m+1}, \dots, \frac{B}{m+1}) \in \mathbb{R}^{m+1}$
\State Initialize $\gamma = (1 - \beta) \gamma_0$
\For{each round $t$ do}
    \State Learn $\pi_t \gets \mathrm{FQI}(r - \lambda_t^{\top} g)$ 
    \State Evaluate $\hat{R}(\pi_t) \gets \mathrm{FQE}(\pi_t, r), \hat{G}(\pi_t)\gets \mathrm{FQE}(\pi_t, g)$
    \State $\hat{\pi}_t \gets \frac{1}{t} \sum_{t'=1}^t \pi_{t'}, \hat{\lambda}_t \gets \frac{1}{t} \sum_{t'=1}^t \lambda_{t'}$
    \State $\hat{R}(\hat{\pi}_t) \gets \frac{1}{t} \sum_{t'=1}^t \hat{R}(\pi_{t'})$
    \State $\hat{G}(\hat{\pi}_t) \gets \frac{1}{t} \sum_{t'=1}^t \hat{G}(\pi_{t'})$
    \State Learn $\tilde{\pi} \gets \mathrm{FQI}(r - \hat{\lambda}_t^{\top} g)$ 
    \State Evaluate $\hat{R}(\tilde{\pi})\gets \mathrm{FQE}(\tilde{\pi}, r), \hat{G}(\tilde{\pi})\mathrm{FQE}(\tilde{\pi}, g)$
    \State $\hat{L}_{\max} = \hat{R}(\tilde{\pi}) - \hat{\lambda}_t^{\top} \left[(\hat{G}(\tilde{\pi}) - \tau)^{\top}, 0\right]^{\top}$
    \State $\hat{L}_{\min} = \underset{\lambda, \Vert \lambda \Vert_1 = B}{\min} \left(\hat{R}(\hat{\pi}_t) - \hat{\lambda}_t^{\top} [(\hat{G}(\hat{\pi}_t) - \tau)^{\top}, 0]^{\top}\right)$
    \If{$\hat{L}_{max} - \hat{L}_{min} \leq \omega$}
        \State \textbf{return} $\hat{\pi}_t$
    \EndIf
    \State Set $z_t \gets \left[ (\tau - \hat{G}(\pi_t))^{\top}, 0 \right]^{\top} \in \mathbb{R}^{m+1}$
    \State $\lambda_{t+1}[i] \gets B \frac{\lambda_t[i]e^{-\eta z_t[i]}}{\sum_j \lambda_t[j]e^{-\eta z_t[j]}} \quad \forall i$
\EndFor
\end{algorithmic}
\end{algorithm}

\begin{proposition}[Theorem 4.4, \citet{le2019batch}]
    \label{end_to_end}
    Let $\pi^*$ be the optimal policy to (\ref{conservative_form}). Denote $\bar{V} = \bar{R} + B\bar{G}$. Let $K$ be the number of iterations of FQE and FQI, $\hat{\pi}$ be the policy returned by algorithm \ref{Algo_contamination}, with termination threshold $\omega$ and robust level $\beta$. For $\epsilon > 0$ and $\delta \in (0, 1)$, when $n = \mathcal{O}(\frac{\bar{V}^4}{\epsilon^2}(\log \frac{K(m + 1)}{\delta} + \dim_{F}\log \frac{\bar{V}^2}{\epsilon^2} + \log \dim_{F}))$, we have with probability at least $1 - \delta$: \begin{equation}
        V^{\hat{\pi}}_r(\mu) \ge V^{\pi^*}_r(\mu) - \omega - \frac{(4 + B)(1 - \beta)\gamma}{(1 - (1 - \beta)\gamma)^3}(\sqrt{\beta_{\mu}}\epsilon + 2((1 - \beta)\gamma)^{K/2}\bar{V}), 
    \end{equation} and \begin{equation}
        V^{\hat{\pi}}_g(\mu) \ge \tau - 2\frac{\bar{V} + \omega}{B} - \frac{((1 - \beta)\gamma)^{1/2}}{(1 - (1 - \beta)\gamma)^{3/2}}(\sqrt{\beta_{\mu}}\epsilon + \frac{2((1 - \beta)\gamma)^{K/2}\bar{V}}{(1 - (1 - \beta)\gamma)^{1/2}})
    \end{equation}
    where $dim_{F}$ is the pseudo-dimension \citep{hastie2009elements} for function approximation class F used in FQI and  FQE , $\beta_{\mu}$ is the concentration coefficient of future state-action distribution \citep{munos2007performance, le2019batch}.
\end{proposition}

\begin{proof}
    The proof is standard as in \citet{le2019batch} up to a different discount factor.
\end{proof}



\section{Experiments}

\label{appendix:3}

In this section, we illustrate the detailed setting of our verification for a solution with R-contamination uncertainty set as in Section \ref{sec:contamination}. Notably, the resulting operator (\ref{contamination_operator_simplified}) for\BestResponse~subroutine coincides with the standard consistency operator, which is used for\BestResponse~subroutine in Constrained RL \citep{le2019batch}, up to a discount factor difference. We thus focus on the joint constrained-robustness comparison as our target, and refer to constrained RL literature \citep{miryoosefi2019reinforcement, le2019batch} for broader comparisons.

\paragraph{Algorithm Design.} We adopt the same choices of\BestResponse~and\OnlineAlgo~with a shortened discount factor in subroutines as in \citet{le2019batch}. In specific, the\BestResponse~algorithm is instantiated with Fitted Q Iteration (FQI) \citep{ernst2005tree}, a model-free off-policy learning approach on the well-collected dataset from \citet{le2019batch} with modifications on constraints. And we use multi-layered CNNs. The\OnlineAlgo~is chosen as Exponentiated Gradient (EG) algorithm \citep{kivinen1997exponentiated}, a variant of online mirror descent. The EG algorithm requires bounded $\lambda$, we thus force $\Vert \lambda \Vert_1 \le B$ which has minor harm on the theoretical analysis as shown in proposition \ref{end_to_end}. To meet the requirement of proposition \ref{meta_prop}, we use the simulator to precisely evaluate value functions of specific policies and Lagrangians, as shown in Algorithm \ref{Algo_contamination}. We set the initial position of our car as fixed, i.e. $\mu = \{s_0\}$, for the accuracy of all the evaluations.

\subsection{Car Racing}

\begin{figure}[h]
    \centering
    \includegraphics[width = .4\textwidth]{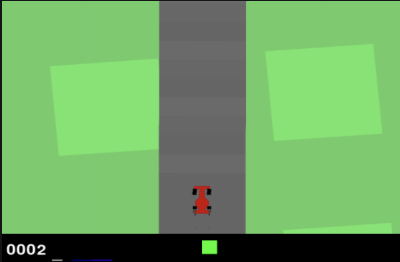}
    \caption{Car Racing environment}
    \label{fig:carracing}
\end{figure}

The environment is chosen as the car racing environment, a high-dimensional domain from Gymnasium \citep{towers_gymnasium_2023}, as shown in figure \ref{fig:carracing}. 
This environment is a racetrack, where each state $s \in \St$ is a $96 \times 96 \times 3$ tensor of raw pixels.
Given each state, the agent has 12 actions $a \in \Ac = \{(i, j, k)|i \in \{-1, 0, 1\}, j \in \{0, 1\}, k \in \{0, .2\}\}$, corresponding to steering angle, amount of gas applied and amount of brake applied, respectively.
In each episode, the goal is to traverse over $95\%$ of the track, measured by the number of tiles which amount to 281 tiles in total. The agent receives a reward of $\frac{1000}{281}$ for passing each single tile and no reward if off-track. A small positive cost of $.1$ applies at every time step with a maximum horizon of $1000$. We further utilize the popular frame-stacking option that is common in practice in online RL for Atari and video games. 

We describe the two constraints we studied as slow driving and edge driving. In slow driving constraint, the agent receives a reward $g_0(s, a) = 1$ if $a$ contains braking action and $0$ otherwise. In edge driving, the agent receives its normalized Euclidean distance of the closest point between the track and the agent from the lane center as its reward $g_1(s, a) \in [0, 10]$. Both constraints may intervene with the primary goal of track traversing. Let $N_t$ be the number of tiles that are traversed by the agent, then the distributionally robust constrained reinforcement learning problem is :\begin{align}
    \max_{\pi \in \Pi} &~\min_{\mathcal{K} \in \otimes_{t \ge 0} \mathcal{P}} \mathbb{E}_{\mathcal{K}} [\sum_{t = 0}^{\infty} \gamma^t (\frac{1000}{281}(N_{t + 1} - N_{t}) - .1) | s_0 = s, \pi]\\
    \mathrm{s.t.}&~ \min_{\mathcal{K} \in \otimes_{t \ge 0} \mathcal{P}} \mathbb{E}_{\mathcal{K}} [\sum_{t = 0}^{\infty} \gamma^t \mathbb{I}(a_t \in \mathcal{A}_{braking}) | s_0 = s, \pi] \ge \tau_0 \notag \\
    &~ \min_{\mathcal{K} \in \otimes_{t \ge 0} \mathcal{P}} \mathbb{E}_{\mathcal{K}} [\sum_{t = 0}^{\infty} \gamma^t \frac{10d(s_t)}{d_{max}} | s_0 = s, \pi] \ge \tau_1. \notag
\end{align}

We finally set the thresholds $\tau = [\tau_0, \tau_1] = [2, 4]$ being close to the constraints satisfaction of baseline ($\gamma = 0.95$), such that the slow driving constraint is satisfied and the edge driving constraint is violated. 

\subsection{Full Results}

In figure \ref{fig:all}, the full results of our car racing experiments are presented.

\begin{figure}[!htb]
    \centering
    \subfigure[Shift of Braking Magnitude]{
        \includegraphics[width = \linewidth]{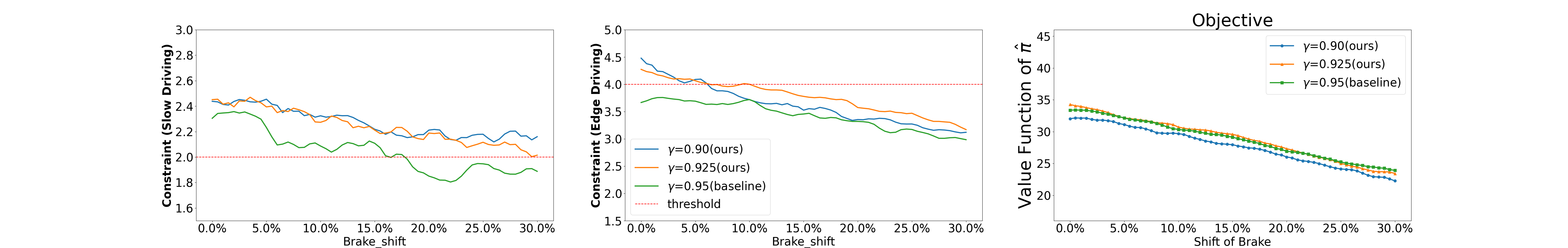}
    }
    \subfigure[Shift of Steering Angle]{
        \includegraphics[width = \linewidth]{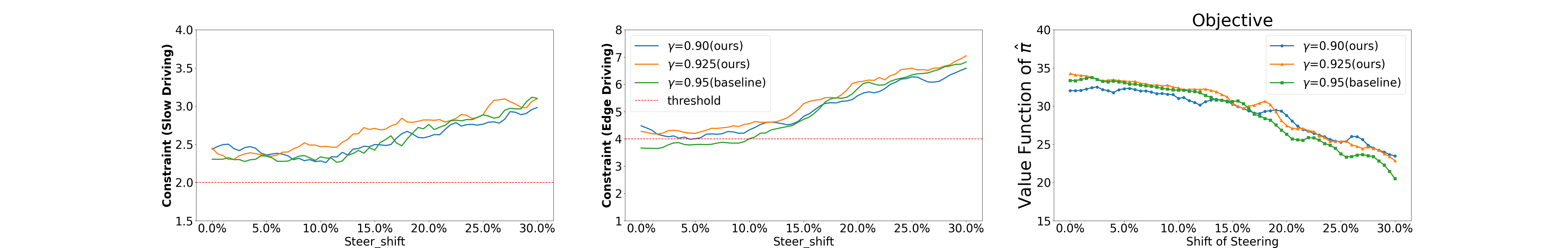}
    }
    \subfigure[Shift of Powew]{
        \includegraphics[width = \linewidth]{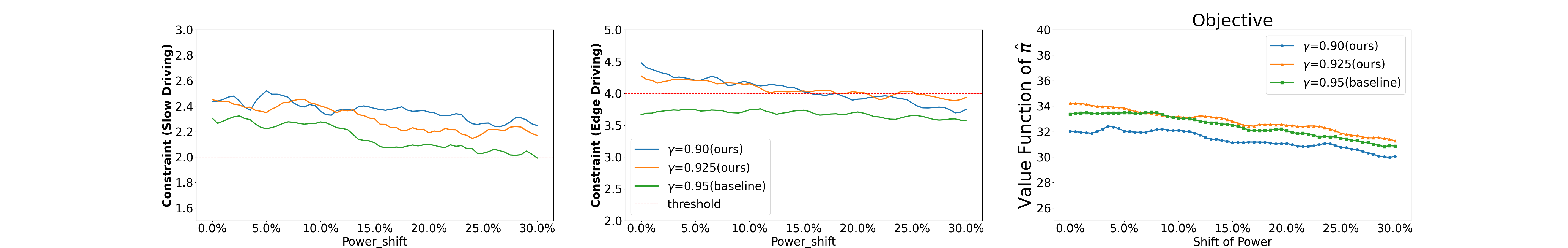}
    }
    \subfigure[Shift of Inertia of Wheel]{
        \includegraphics[width = \linewidth]{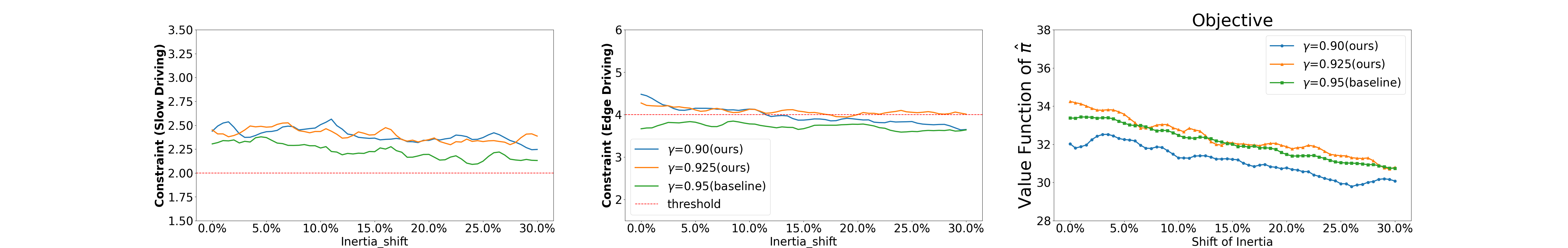}
    }
    \vfill
    \label{fig:all}
    \caption{Full results with Four different shifts: Higher is better, Left two are constraints and the right one is the objective. The bar graphs of constraints satisfaction are directly produced from these results. }
\end{figure}